\documentclass[acmsmall, screen]{acmart}

\AtBeginDocument{%
	}

\usepackage{amsmath,amsfonts,amsthm}
\usepackage{graphicx}
\usepackage{textcomp}
\usepackage{enumerate}
\usepackage{wrapfig}
\usepackage{float}

\usepackage{comment}
\usepackage{comment}
\usepackage{listings,color}
\usepackage{multirow}
\usepackage{balance}
\usepackage{url}
\usepackage{array,color,colortbl,graphicx,multirow}
\usepackage{subcaption}
\usepackage{caption}
\usepackage[toc,page]{appendix}
\usepackage{makecell}
\usepackage{url}
\usepackage{nicematrix}
\usepackage{mathtools}
\usepackage{tikz}
\usetikzlibrary{shapes}
\usetikzlibrary{decorations.markings,shadows, shapes}
\usepackage[linesnumbered,ruled,vlined]{algorithm2e}
\SetKwInput{KwData}{Input}
\SetKwInput{KwResult}{Output}
\SetKwComment{Comment}{/* }{ */}
\usepackage{wrapfig}
\usepackage{enumitem}

\usepackage{etoolbox}
\usepackage{tikz}
\usetikzlibrary{patterns, positioning,shapes.misc,matrix, arrows.meta, calc} %
\usepackage[most]{tcolorbox}  
\usepackage{xcolor} 
\usepackage{siunitx}

\newcommand{\name}{Aurora\xspace}
\newcommand{\para}[1]{\noindent{\textbf{#1}}\:}
\newcommand{\moe}{MoE\xspace}
\newcommand{\bomp}{bottleneck matching problem\xspace}

\newcommand{\uGbps}[1]{\SI{#1}{Gbps}}

\author{Jialong Li, Shreyansh Tripathi, Lakshay Rastogi, Yiming Lei, Rui Pan, \\Yiting Xia}

\authorsaddresses{}

\renewcommand{\shortauthors}{Li et al.}

\renewcommand\footnotetextcopyrightpermission[1]{} 
\begin{document}

\title[\name]{Optimizing Mixture-of-Experts Inference Time Combining Model Deployment and Communication Scheduling}

\begin{abstract}

As machine learning models scale in size and complexity, their computational requirements become a significant barrier. Mixture-of-Experts (\moe) models alleviate this issue by selectively activating relevant experts. Despite this, \moe models are hindered by high communication overhead from all-to-all operations, low GPU utilization due to the synchronous communication constraint, and complications from heterogeneous GPU environments.

This paper presents \name, which optimizes both model deployment and all-to-all communication scheduling to address these challenges in \moe inference. \name achieves minimal communication times by strategically ordering token transmissions in all-to-all communications. It improves GPU utilization by colocating experts from different models on the same device, avoiding the
limitations of synchronous all-to-all communication.
We analyze \name's optimization strategies theoretically across four common GPU cluster settings:
exclusive vs. colocated models on GPUs, and homogeneous vs. heterogeneous GPUs.
\name provides optimal solutions for three cases, and for the remaining NP-hard scenario, it offers a polynomial-time sub-optimal solution with only a 1.07$\times$ degradation from the optimal.

\name is the first approach to minimize \moe inference time via optimal model deployment and communication scheduling across various scenarios. Evaluations demonstrate that \name significantly accelerates inference, achieving speedups of up to 2.38$\times$ in homogeneous clusters and 3.54$\times$ in heterogeneous environments. Moreover, \name enhances GPU utilization by up to 1.5$\times$ compared to existing methods.

\end{abstract}

\maketitle

\section{Introduction}\label{sec:intro}

Serving deep learning and large language models has become increasingly critical as they are integrated into a wide range of online applications, such as programming assistance, search engines, and conversational bots. However, as the size and complexity of these models continue to grow, it is challenging to meet the high computational demands and stringent latency requirement.

Mixture-of-Experts (\moe) models offer an effective solution to reduce computational demands while preserving performance. They achieve this by dynamically activating only a subset of specialized components, known as \textit{experts}, for input \textit{tokens}. This selective activation reduces the overall computational load without sacrificing efficiency and accuracy. By engaging only the most relevant experts for specific tasks, \moe models optimize resource utilization and processing speed.

Despite the considerable benefits, inference of MoE models still faces significant challenges.
The most prominent issue is \textit{high communication overhead}. The all-to-all communication pattern in MoE models, identified as a major bottleneck~\cite{rajbhandari2022deepspeed, he2022fastermoe, huang2023towards}, is largely due to the dynamic selection of experts. This results in uneven data exchange among GPUs, leading to network bandwidth contention and prolonged communication times.

Moreover, MoE models suffer from \textit{low GPU utilization}. This problem arises because all-to-all communication is typically implemented using synchronous operations~\cite{li2023accelerating, shen2022se, Flexmoe, Janus, moesys, Lazarus, Schemoe}. As a result, GPUs hosting unpopular experts remain idle while waiting for communication to complete on GPUs handling popular experts.

Lastly, \textit{GPU heterogeneity}, which is common due to incremental deployments, adds further complexity to MoE model deployment~\cite{mlaas, chen2022ta, heterog, gavel}. The varied hardware configurations complicate the efficient allocation and utilization of resources across the model. To fully harness the potential of MoE models, these challenges need to be effectively addressed.

Existing solutions fail to solve the problem from all fronts. Most approaches either reduce communication overhead by balancing token loads~\cite{fedus2022switch, lepikhin2020gshard, riquelme2021scaling, child2019generating, shazeer2017outrageously, hwang2023tutel, he2022fastermoe, chen2022ta, huang2023towards, hwang2023pre, Hetumoe} or by accelerating the all-to-all operation~\cite{he2022fastermoe, hwang2023tutel, rajbhandari2022deepspeed, shoeybi2019megatron, rajbhandari2020zero, he2021fastmoe, li2023accelerating, Schemoe, moesys}, but still struggle with low GPU utilization. Other approaches pack multiple experts from the same model on a single GPU to reduce idle time~\cite{Flexmoe, Lazarus, Prophet, huang2023towards}, but these experts remain blocked by synchronous all-to-all communication, preventing full interleaving of computation and communication. Besides, these methods rely on empirical approaches, lacking theoretical backing, and are designed for specific settings, failing to account for the diverse configurations of production GPU clusters, such as heterogeneous hardware.

In this paper, we propose \textit{\name, a comprehensive solution for minimizing the inference time of MoE models}. Our design combines expert colocation, GPU assignment, and communication scheduling, supported by theoretical analysis across four distinct GPU cluster settings based on two key dimensions: exclusive vs. colocated experts on GPUs, and homogeneous vs. heterogeneous GPUs. \name achieves \textit{optimal} inference time in most cases, except for the NP-hard scenario of colocating experts on heterogeneous GPUs, where we provide a \textit{sub-optimal} polynomial-time solution with inference time only 1.07$\times$ the optimum, as shown in our simulations.

To the best of our knowledge, \name offers the first theoretical derivation of minimal MoE inference time. Our key insights can guide the development of future MoE inference systems: minimal all-to-all communication time is achieved by ordering token transmission to avoid bandwidth contention; in homogeneous clusters, minimizing inference time is equivalent to minimizing communication time; for exclusive experts on heterogeneous GPUs, assigning experts by token load to GPUs in descending capacity minimizes inference time; and the NP-hard case of colocating experts on heterogeneous GPUs is a 3-dimensional matching problem, which can be approximated by decoupling it into two dependent bipartite graphs.

Extensive simulations demonstrate the effectiveness of \name. Using production MoE inference traces from Google, \name reduces inference time by up to 2.38$\times$ in homogeneous GPU clusters and up to 3.54$\times$ in heterogeneous clusters. By colocating experts from different models, \name also improves GPU utilization by up to 1.5$\times$ compared to state-of-the-art solutions that colocate experts from the same model.
Even with inaccurate inputs for \name's optimization, with up to 75\% noise in model statistics, inference time is extended by only 15.8\%.

\section{Preliminaries}\label{sec:prelim}

In this section, we first explore the structure of \moe inference to understand how the different components work together within the model ($\S$\ref{sec:moe_inference}). Next, we discuss the distinctive features of \moe inference that set it apart from other architectures ($\S$\ref{sec:moe_characteristics}). We then identify the key bottlenecks that affect \moe inference performance ($\S$\ref{sec:moe_bottleneck}). Finally, we outline the essential prerequisites required for \name ($\S$\ref{sec:moe_prerequisites}).

\begin{figure}[tb]
    \centering
    \includegraphics[width=0.70\linewidth]{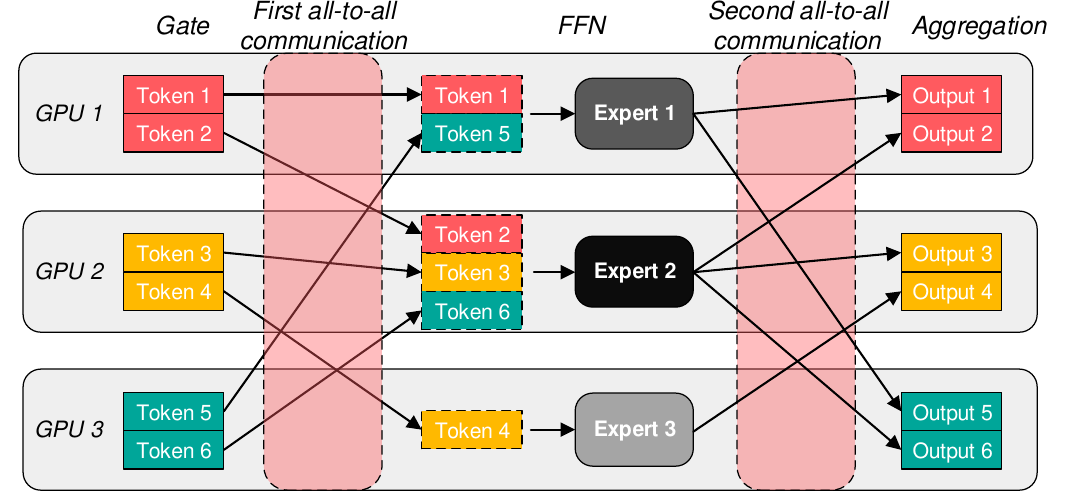}\vspace{-0.10in}
    \caption{\moe model structure.}\label{fig:moe}
    \vspace{-0.15in}
\end{figure}

\subsection{MoE Inference}\label{sec:moe_inference}

An MoE model comprises multiple MoE layers. For \moe training, each layer involves both a forward and a backward pass, while inference requires only the forward pass. Fig.~\ref{fig:moe} illustrates the process of an \moe layer, highlighting the separation of \textit{computation} and \textit{communication phases}. The computation phase consists of three components: the gate function, the feed-forward network (FFN), and aggregation. Two all-to-all communications occur during the communication phase. These two all-to-all communications are opposite in terms of data flows.

\para{Gate.} The gate network determines which experts should be activated for the input tokens. In general, each token will be sent to one or two experts. 

\para{FFN.} An FFN is typically an expert. Each expert is responsible to process the tokens assigned by the gate network.

\para{Aggregation.} This operation reshapes the tensors and computes the weighted output. After aggregation, the process proceeds to the next \moe layer.

\para{First all-to-all communication.} The first all-to-all communication occurs after the gate network. During this process, each token is dispatched to the assigned experts.

\para{Second all-to-all communication.} The second all-to-all communication is for exchanging outputs of experts, ensuring the original sequences are organized before the start of next layer.

\subsection{Characteristics of MoE Inference}\label{sec:moe_characteristics}

Here, we outline three key characteristics of \moe inference, which shed light on the inference bottlenecks discussed in $\S$\ref{sec:moe_bottleneck}.

\para{Synchronous all-to-all communications.} In this process, all-to-all communication is synchronous, meaning that computation (including FFN and aggregation) can only begin once every GPU has completed data transmission. This leads to the GPU computation resource idleness.

\para{Reversed all-to-all communications.} Within the same forward pass, the two all-to-all communications are reversed. For each data transfer from GPU $i$ to $j$ in the first communication, there is a corresponding data transfer from GPU $j$ to $i$ in the second. The data sizes in these transfers are identical, as the FFN architecture ensures that the input and output data sizes are the same.

\para{Non-overlapping communication and computation.} Communication and computation processes do not overlap; each step can only commence after the previous one is completed.

\subsection{MoE Inference Bottlenecks}\label{sec:moe_bottleneck}

\para{High communication overhead.} Existing research has identified all-to-all communication as a significant bottleneck in \moe inference~\cite{rajbhandari2022deepspeed, he2022fastermoe}. A recent study~\cite{huang2023towards} highlights that the all-to-all communication can constitute over 60\% of inference time when using four GPUs, and the overhead increases substantially with additional GPUs . 

The high communication overhead arises from two main factors. First, the dynamic selection of experts results in an uneven distribution of tokens among GPUs~\cite{he2022fastermoe, radford2018improving, child2019generating}, leading to some GPUs being heavily loaded while others remain idle. Second, the all-to-all communication is typically implemented using synchronous operations~\cite{li2023accelerating}. These operations are inefficient as they cause resource wastage when either communication or computation is not fully utilized. This inefficiency is further exacerbated by the dynamic nature of expert selection.

\para{Low GPU utilization.} Uneven load distribution and synchronous all-to-all communication also contribute to low GPU utilization. GPUs supporting unpopular experts remain idle most of the time~\cite{hwang2023pre}. A study of GPU cluster data from Alibaba reveals that less than 10\% of GPUs reach 80\% utilization~\cite{Antman}. 

\para{GPU cluster heterogeneity.} In production clusters, GPU heterogeneity is common due to incremental deployments and rapid advancements in GPU design~\cite{mlaas, gavel, heterog, sia, Hap}. These clusters feature varied hardware configurations, including different types of GPUs and diverse resource setups. This heterogeneity complicates the deployment of \moe models and must be considered to optimize performance.

\subsection{Prerequisites in \name}\label{sec:moe_prerequisites}

Before we delve into the details of each scenario, let's outline the key prerequisites for this work.

\textit{Each GPU hosts at most two models.}
As shown in Fig.~\ref{fig:moe},
MoE inference involves alternating computation and communication phases separated by clear barriers.
Colocating two models on a GPU allows them to efficiently interleave resource usage---one model performs computation while the other uses the network. Adding a third model, however, forces one to wait for resource access, leading to increased inference time.

\textit{The network is represented by a big switch model.} \moe inference typically requires several to dozens of GPUs, often housed within a single rack and connected by a high-performance network. As illustrated in Fig.~\ref{fig:network}(a), this non-blocking network fabric can be modeled as a big switch, which interconnects GPUs enabling low-latency and high-throughput communication between them.

\textit{The optimization is based on historical statistics of \moe models.}
Inference service providers usually collect statistics of MoE models for performance monitoring and troubleshooting, such as the token distribution across GPUs and the average computation times for the Gate network, FFN, and Aggregation operations. \name uses such historical data to guide optimization, and 
this work focuses on theoretical analysis of our optimization mechanisms based on these precise inputs. As shown in our simulations ($\S$\ref{sec:evaluation}), even with up to 75\% unpredictable inference requests after the optimization plan is deployed, the inference time of \name is only degraded by 15.8\%.

\section{Overview}\label{sec:overview}

\begin{figure}[tb]
    \centering
    \includegraphics[width=0.98\linewidth]{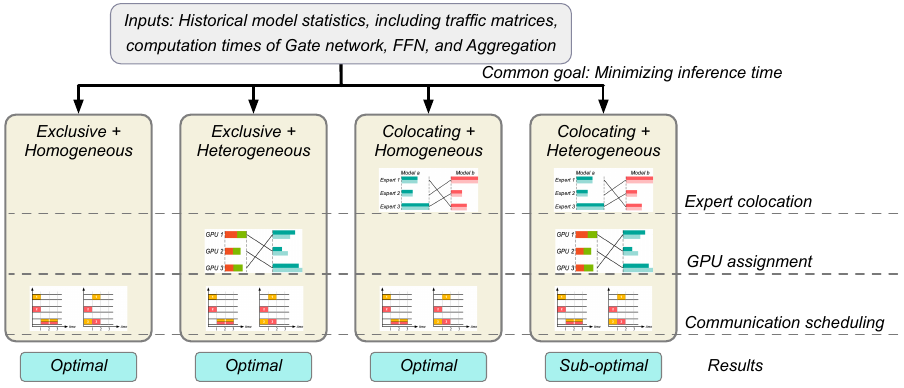}\vspace{-0.15in}
    \caption{\name aims to minimize inference time across four different scenarios. It optimizes expert colocation, GPU assignment, and communication scheduling for each case. \name achieves optimal results in the first three scenarios and delivers suboptimal performance in the final one due to its NP-hardness.}\label{fig:overview}
    \vspace{-0.2in}
\end{figure}

In this section, we begin by outlining \name's inputs and optimization goal across four scenarios of GPU cluster settings. Next, we explain how expert colocation, GPU assignment, and communication scheduling impact inference time. Finally, we provide a summary of each scenario.

\para{Inputs.}
As discussed in $\S$\ref{sec:moe_prerequisites}, we use historical model statistics to guide decisions on expert colocation, GPU assignment, and communication scheduling. At a high level, the inputs include traffic matrices of token distribution across GPUs during each all-to-all communication, as well as computation times for the Gate network, FFN, and Aggregation operations.
The detailed input parameters are listed in Table~\ref{tab:input} and explained in $\S$\ref{sec:exclusive_homo}.

\begin{figure}[tb]
    \centering
    \includegraphics[width=0.98\linewidth]{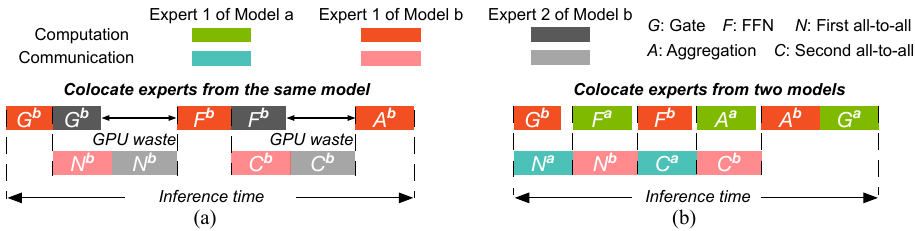}\vspace{-0.10in}
    \caption{(a) Colocating experts from the same model results in wasted GPU resources and increased inference time, as follow-up computations are delayed by synchronous all-to-all communications. (b) Colocating experts from different models enables full interleaving of computation and communication, resolving this issue.
    }\label{fig:colocating_compare}
    \vspace{-0.1in}
\end{figure}

\begin{figure}[tb]
    \centering
    \includegraphics[width=0.98\linewidth]{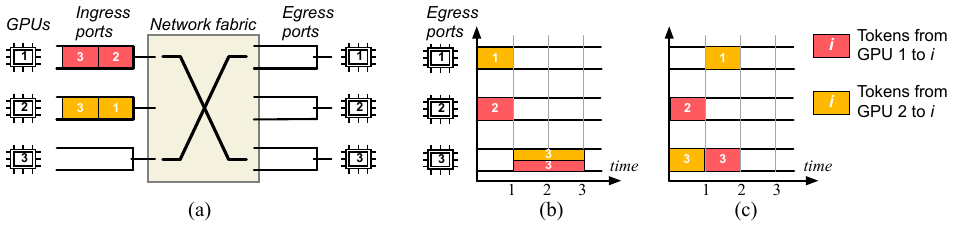}\vspace{-0.10in}
    \caption{(a) A big switch model representing the non-blocking inter-GPU network fabric. (b) Originally, the all-to-all communication of tokens from GPU 1 (red) and GPU 2 (yellow) to all other GPUs takes 3 units of time overall. (c) Optimizing the token order reduces transmission time to 2 units.}\label{fig:network}
    \vspace{-0.15in}
\end{figure}

\para{Optimization goal.}
\name is designed to minimize the inference time of \moe models. Given the diverse settings of modern GPU clusters, we analyze two key dimensions: exclusive GPU usage per model vs. colocating models on the same GPUs, and homogeneous vs. heterogeneous GPU types. Across the four combinations of these dimensions, as shown in Fig.~\ref{fig:overview}, \name achieves optimal performance in the first three scenarios. We prove the last scenario to be NP-hard and propose a sub-optimal solution with inference time only 1.07$\times$ the optimum, based on our evaluation in $\S$\ref{sec:evaluation}.
When possible, we colocate MoE models on GPUs to maximize GPU utilization at best effort.

\para{Expert colocation.}
\name colocates experts on the same GPU to maximize utilization, where applicable, as shown in the colocating scenarios in Fig.~\ref{fig:overview}. As motivated in $\S$\ref{sec:moe_prerequisites}, \name colocates up to two experts per GPU, interleaving their computation and communication. Previous studies colocated experts from the same model~\cite{li2023accelerating}, which wastes GPU resources and extends inference time. As shown in Fig.~\ref{fig:network}(a), this is because experts from the same model are bound by the synchronous all-to-all communication, delaying subsequent computation phases like FFN and Aggregation.

In contrast, \name colocates experts from different models.
To maximize GPU utilization and minimize inference time, it identifies the optimal combination of experts that
complement each other in terms of computation and communication needs. As illustrated in Fig.~\ref{fig:network}(b), the two experts take turns to use available computation and communication resources. \name pairs a computation-intensive expert with a communication-intensive one to efficient use of GPUs.

\para{GPU assignment.}
Heterogeneous clusters, as shown in Fig.~\ref{fig:overview}, require selecting the appropriate GPU types for experts. For instance, deploying a popular expert on a high-performance GPU with high FLOPS, memory capacity, and network bandwidth helps minimize computation and communication times. \name assigns experts to suitable GPU types without worrying about the deployment to specific GPU IDs, as GPUs of the same type are interchangeable when connected through the ``big switch'' model of a non-blocking network, as discussed in $\S$\ref{sec:moe_prerequisites}.

\para{Communication scheduling.}
All four scenarios in Fig.~\ref{fig:overview} require communication scheduling to reduce the communication time, which involves determining the order of token transmission in the all-to-all communications.
Different transmission orders can lead to varying communication times. For instance, in Fig.~\ref{fig:network}(a), GPU 1 sends tokens to GPUs 2 and 3, while GPU 2 sends to GPUs 1 and 3. In Fig.~\ref{fig:network}(b),  communication takes 3 units of time when GPU 1 first sends to GPU 2 and then to GPU 3, and GPU 2 sends to GPU 1 and then to GPU 3. However, as shown in Fig.~\ref{fig:network}(c), changing GPU 2’s transmission order to send to GPU 3 first, then to GPU 1, reduces communication time to 2 units.
In practice, reordering token transmission can be achieved with a buffer layer at the computation operations, which calls communication collective libraries, such as NCCL, in the desired order.

We summarize the main results of our theoretical analysis of each scenario as follows. 

\para{Exclusive + Homogeneous} ($\S$\ref{sec:exclusive_homo}). This scenario considers running models exclusively on clusters where all GPUs have identical computing power and network bandwidth. Theorem~\ref{thm:exclusive_homo_infertime} proves that minimizing inference time is equivalent to minimizing communication time. Theorem~\ref{thm:exclusive_homo_comm_time} further shows that communication time is minimized by ordering token transmission to avoid bandwidth contention at the receiving GPUs. The minimum communication time is determined by the GPU handling the largest traffic volume, whether sending or receiving. Alg.~\ref{alg:order} ($\S$~\ref{sec:exclu_homo_2}) provides the algorithm for finding the optimal order that minimizes inference time.

\para{Exclusive + Heterogeneous} ($\S$\ref{sec:exclusive_hetero}). This scenario tackles the challenges of running models exclusively on GPUs with different computing power and network bandwidth. Theorem~\ref{thm:exclusive_hetero_assignment} demonstrates that sorting experts by token load and assigning them to GPUs in descending order of performance minimizes inference time. Theorem~\ref{thm:exclusive_hetero_comm_time} proves that the transmission order for homogeneous clusters (Theorem~\ref{thm:exclusive_homo_comm_time}) also minimizes communication time in a heterogeneous setting.

\para{Colocating + Homogeneous} ($\S$\ref{sec:colocating_homo}). This scenario examines improving GPU utilization by colocating two \moe models. The colocation strategy affects the aggregated communication time of the two models, and thus, their overall inference time. Theorem~\ref{thm:colocating_homo_comm_time} shows that minimizing the aggregated communication time leads to optimal overall inference time. To this end, we solve the bottleneck matching problem to find the optimal expert colocation, thereby minimizing communication time and achieving optimal inference time.

\para{Colocating + Heterogeneous} ($\S$\ref{sec:colocating_hetero}). Extending the colocation strategy to heterogeneous clusters involves communication scheduling, GPU assignment, and expert colocation, making it the most complex scenario. We model it as a 3-dimensional matching problem, which is NP-hard. By decoupling the matching into two dependent bipartite graphs, we propose a sub-optimal but effective solution, which prolongs the inference time by only 1.07$\times$ compared to the optimum.

\vspace{-0.05in}
\section{Exclusive Models on Homogeneous Clusters}\label{sec:exclusive_homo}

\begin{figure*}[tbp] %
    \centering
    \begin{minipage}[tb]{0.42\textwidth}
        \footnotesize
        \centering
        \setlength{\tabcolsep}{1pt}
        \captionof{table}{Input parameters.}
        \vspace{-0.15in}
        \begin{tabular}{cl}
            \toprule
            Symbol  &   Explanation   \\ %
            \midrule
            $n$                    &  Number of experts      \\ %
            $\mathbb{D}_N$ &  Traffic matrix (token distribution) of the first all-to-all   \\ 
            $\mathbb{D}_C$  &  Traffic matrix (token distribution) of the second all-to-all   \\
            $d_{ij}$                  & Data (tokens) sent from GPU $i$ to $j$        \\ %
            $B_i$                &   Bandwidth of GPU $i$       \\ %
            $|G_i^a|$                &  Computation time of Model $a$'s Gate on GPU $i$ \\
            $|F_i^a|$                &  Computation time of Model $a$'s FFN on GPU $i$  \\ %
            $|A_i^a|$                &  Computation time of Model $a$'s Aggregation on GPU $i$  \\ %
            \bottomrule
        \end{tabular}
        \label{tab:input}\vspace{-0.15in}
    \end{minipage}%
    \hfill
    \begin{minipage}[tb]{0.45\textwidth}
        \centering
        \includegraphics[width=\linewidth]{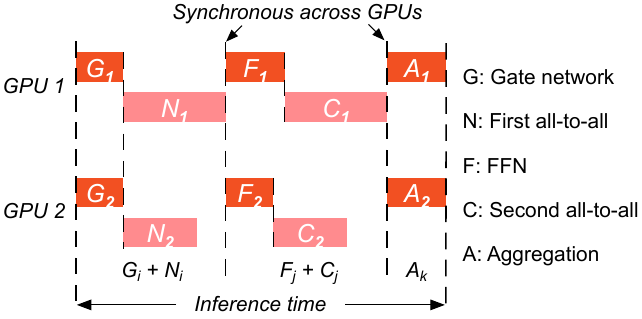}\vspace{-0.1in}
        \caption{Running MoE models exclusively on homogeneous clusters.}
        \label{fig:single_homo}
        \vspace{-0.15in}
    \end{minipage}
\end{figure*}

In this section, we derive the minimum inference time for running models exclusively on homogeneous clusters, referred to as the Exclusive + Homogeneous scenario.

\para{Input parameters.} Table~\ref{tab:input} lists the input parameters used by all four scenarios. For a specified MoE model consisting of $n$ experts, each expert is placed on one GPU, requiring a total of $n$ GPUs. The token distribution in each layer is known in advance  is represented by a traffic matrix $\mathbb{D}_N$ for the first all-to-all communication, and $\mathbb{D}_C$ for the second all-to-all communication. Note that $\mathbb{D}_N$ and $\mathbb{D}_C$ are reversed as we state in $\S$\ref{sec:moe_characteristics}. The matrix is an $n \times n$ matrix with elements $d_{ij}$, indicating the traffic sent from GPU $i$ to $j$. The symbol $|G_i^a|$, $|F_i^a|$, and $|A_i^a|$ represent the computation times of Model $a$'s Gate, FFN, and Aggregation components, respectively, on GPU $i$.

\para{Solution overview.} We first prove that in the Exclusive + Homogeneous scenario, minimizing inference time is equivalent to minimizing communication time ($\S\ref{sec:exclu_homo_1}$). Next we show how to determine the transmission order to achieve minimal communication time ($\S\ref{sec:exclu_homo_2}$).

\subsection{Minimizing inference time equals minimizing communication time}\label{sec:exclu_homo_1}

\begin{theorem}\label{thm:exclusive_homo_infertime}
In the Exclusive + Homogeneous scenario, minimizing inference time is equivalent to minimizing communication time.
\end{theorem}

\begin{proof}
This proof is straightforward because, in the Exclusive + Homogeneous scenario, the only factor influencing inference time is the communication scheduling.

We first derive the inference time expression. As shown in Fig.~\ref{fig:single_homo}, the two all-to-all communications are synchronous across GPUs. This synchronization means that the FFN and Aggregation processes can only start after the last data flow is complete. The inference time is therefore divided into three parts: the Gate and the first all-to-all communication, the FFN and the second all-to-all communication, and the Aggregation. Due to the strict barrier between each layer, the inference time for a layer is determined by the slowest GPU. So the inference time is determined by summing the maximum values of each part, as represented by the following equation.

\vspace{-0.15in}
\begin{equation}\label{eqn:single_homo_1}
\textit{Inference time } \;\; t = \max(|G_i| + |N_i|) + \max(|F_j| + |C_j|) + \max(|A_k|),\; i,j,k \in [1, n]
\end{equation}
\vspace{-0.15in}

In Eqn.~\ref{eqn:single_homo_1}, the symbols $|G_i|$, $|F_i|$, and $|A_i|$ indicate the duration of the Gate, FFN, and Aggregation processes on GPU $i$. The symbols $i, j, k$ each represent different GPUs, as the maximum processing time for each part can occur on different GPUs. For this scenario, we can make the following three observations.

\begin{enumerate}[label=(\arabic*), leftmargin=3em, topsep = 0.5em]
    \item The assignment of GPUs to experts in homogeneous clusters requires no special decisions, as all GPUs possess identical computational power and network bandwidth.
    \item The computation times for Gate processes are equal across all GPUs, and the same applies to Aggregation.
    \item The computation time for the FFN is determined by the number of tokens it processes, with more tokens resulting in longer computation times.
\end{enumerate}

With observations (1) and (2), we have $|G_i| = |G|$, $|A_k| = |A|$. Based on observation (3), we can state that $\max(|F_j| + |C_j|) = \max(|F_j|) + \max(|C_j|)$, since $|F_j|$ and $|C_j|$ increase simultaneously. Thus, Eqn.~\ref{eqn:single_homo_1} can be expressed as follows.

\vspace{-0.15in}
\begin{equation}\label{eqn:single_homo_2}
t = |G| + \max(|N_i|) + \max(|F_j|) + \max(|C_j|) + |A|,\; i,j \in [1, n]
\end{equation}
\vspace{-0.15in}

As discussed in $\S$\ref{sec:moe_characteristics}, the two all-to-all communications are reversed. The GPU receiving the highest volume of data at $\mathbb{D}_N$ is also the one transmitting the largest amount at $\mathbb{D}_C$. Consequently, $\max(|N_i|)$ and $\max(|C_j|)$ occur on the same GPU. Therefore, Eqn.~\ref{eqn:single_homo_2} can be further expressed as:

\vspace{-0.15in}
\begin{equation}\label{eqn:single_homo_3}
t = |G| + \max(|N_i|) + \max(|F_i|) + \max(|C_i|) + |A|,\; i \in [1, n]
\end{equation}
\vspace{-0.15in}

In Eqn.~\ref{eqn:single_homo_3}, $\max(|F_i|)$ represents the computation time of the FFN processing the highest number of tokens, which is constant regardless of its deployment. Therefore, to achieve optimal inference time, the remaining task is to minimize $\max(|N_i|)$, the first all-to-all communication time.
\end{proof}

\subsection{Scheduling transmission order to minimize communication time}\label{sec:exclu_homo_2}

Fig.~\ref{fig:network}(a)-(c) show that the communication time depends on the order in which tokens are transmitted. Intuitively, the time for an all-to-all communication cannot be less than the time it takes for the GPU with the heaviest traffic to send or receive its tokens. For example, suppose GPU $i$ receives the largest amount of traffic, denoted by $d$, and the bandwidth is $B$. This means GPU $i$ will need at least $d/B$ time to receive all tokens. The question is, can we design a transmission order that completes the all-to-all communication in exactly $d/B$ time? The following theorem provides an affirmative answer.

\begin{theorem}\label{thm:exclusive_homo_comm_time}
The communication time is minimized by transmitting tokens in an order that avoids bandwidth contention at the receiving sides. The minimum communication time is $b_{max}$ = $max(\sum_{j=1}^{n}d_{ij},\; \sum_{i=1}^{n}d_{ij}) / B$. 
\end{theorem}

Theorem~\ref{thm:exclusive_homo_comm_time} establishes that GPU should avoid sending tokens to the same destination simultaneously. Fig.~\ref{fig:network}(c) presents an optimal order that avoids bandwidth contention at the receiving sides. This order guarantees that at any time, each GPU only receives tokens from one GPU. Theorem~\ref{thm:exclusive_homo_comm_time} also shows that the minimum communication time is determined by the maximum column or row sum in the traffic matrix $\mathbb{D}$\footnote{We will remove the elements in the diagonal of matrix $\mathbb{D}$ as the source and destination are the same.}. In other words, if the largest amount of data being sent or received on a single GPU is $d$, then the entire all-to-all communication can be completed in $d/B$ time. A sketch of the proof is provided below, with the detailed proof available in Appx.~\ref{appendix:theorem1}.

\begin{proof}[Sketched Proof]

In homogeneous clusters, we set $B$ to 1 for simplification. The proof involves transforming the traffic matrix $\mathbb{D}$ into $\mathbb{D'}$ by adding artificial traffic matrix $\mathbb{X}$ with non-negative values. With the updated matrix $\mathbb{D'}$, it ensures that the sum of each column or row equals $b_{max}$. We then demonstrate the all traffic in $\mathbb{D'}$ can be transmitted within the time $b_{max}$, by constructing a transmission order where GPUs do not send tokens to the same destination simultaneously. Since $\mathbb{D'}$ is constructed by augmenting the original traffic matrix $\mathbb{D}$ with the non-negative traffic matrix $\mathbb{X}$, the time required for transmitting traffic in $\mathbb{D}$ cannot exceed $b_{max}$. The optimal transmission order can also be obtained by removing artificial traffic from $\mathbb{D'}$.

Moving forward, our approach unfolds in three key steps. Initially, we illustrate the conversion of the traffic matrix $\mathbb{D}$ into $\mathbb{D'}$ by incorporating matrix $\mathbb{X}$. Subsequently, we prove that the minimum communication time for $\mathbb{D'}$ is $b_{max}$. Finally, we prove the existence of a non-negative $\mathbb{X}$.

\indent \para{\textbf{1. Convert $\mathbb{D}$ to $\mathbb{D'}$ by adding non-negative $\mathbb{X}$}}

\begin{itemize}
\renewcommand{\labelitemi}{\scriptsize$\bullet$}
    \item Construct $\mathbb{D'}$ by adding non-negative artificial traffic matrix $\mathbb{X}$ to $\mathbb{D}$: $\mathbb{D}$ + $\mathbb{X}$ = $\mathbb{D'}$.
    \item Ensure for each row $\sum_{j=1}^{n}d'_{ij} = b_{max}$, and for each column $ \sum_{i=1}^{n}d'_{ij} = b_{max}$, $d'_{ij} \in$ $\mathbb{D'}$.
\end{itemize}

\indent \para{\textbf{2. Prove the minimum communication time for $\mathbb{D'}$ is $b_{max}$}}

\begin{itemize}
\renewcommand{\labelitemi}{\scriptsize$\bullet$}
    \item In each time slot, each GPU sends and receives exactly one token.
    \item Demonstrate that each GPU can send and receive data without interruption. As a result, all GPUs complete their communication within $b_{max}$, making the all-to-all communication time equal to $b_{max}$.
\end{itemize}

\indent \para{\textbf{3. Prove the existence of non-negative $\mathbb{X}$}}
\begin{itemize}
\renewcommand{\labelitemi}{\scriptsize$\bullet$}
    \item Transform the $n \times n$ matrix $\mathbb{X}$ to an $n^{2} \times 1$ vector $\mathbf{x}$.
    \item Formulate the problem using the system of equations: $\mathbb{A} \mathbf{x} = \Delta \mathbf{b}$.
    \item Use Farkas' Lemma~\cite{farkas} to show that a non-negative solution $\mathbf{x}$ exists, which implies the existence of $\mathbb{X}$.
\end{itemize}
\vspace{-0.15in}
\end{proof}

\para{Determining the token transmission order.} We show how to establish the token transmission order for each GPU, with the input of traffic matrix $\mathbb{D}$. According to Theorem~\ref{thm:exclusive_homo_comm_time}, data transmission at the bottleneck GPU should be continuous. Therefore, we first determine the order at the bottleneck.

As shown in Alg.~\ref{alg:order}, we begin by identifying the bottleneck GPU, the one handling the most traffic. The transmission order at the bottleneck can be chosen randomly. After establishing this, we remove the traffic from $\mathbb{D}$ (\textit{Lines 1--4}). For the remaining GPUs, we sort them based on their traffic load in descending order (\textit{Line 6}) and arrange the token transmission to avoid conflicts with the existing order (\textit{Line 8}). The order follows the pattern illustrated in Fig.~\ref{fig:theorem_1}(b), Appx.~\ref{appendix:theorem1}. We continue to remove traffic and update $\mathbb{D}$ accordingly (\textit{Line 10}). This process repeats until $\mathbb{D}$ is empty, resulting in a token transmission order that meets the requirements of Theorem~\ref{thm:exclusive_homo_comm_time}.

In summary, with Theorem~\ref{thm:exclusive_homo_comm_time} we can derive the optimal communication time, and further obtain the the minimum inference time with Eqn.~\ref{eqn:single_homo_3}.

\begin{algorithm}[tb]
    \small
	\caption{Determine token transmission order}\label{alg:order}
	\KwData{All-to-all traffic matrix $\mathbb{D}$}
    \KwResult{Token transmission order $\mathcal{O}$}
	Set $\mathcal{O} \leftarrow \emptyset$;\\
    Find the bottleneck GPU (with the most traffic) \\
    Choose a random order for tokens at the bottleneck, add to $\mathcal{O}$ \\
    Remove the bottleneck traffic from $\mathbb{D}$\\
	\While{$\mathbb{D}$ is not empty}{
		Sort GPUs by traffic amount in descending order \\
        \For{Each GPU $i$ in sorted list}{
            Arrange tokens to avoid conflicts with existing order in $\mathcal{O}$\\
            Add the new order for GPU $i$ to $\mathcal{O}$ \\
            Remove traffic handled by GPU $i$ from $\mathbb{D}$\\
        }
	}
    return $\mathcal{O}$
\end{algorithm}

\begin{tcolorbox}[colback=blue!3, colframe=blue!30!black, boxrule=0.25mm, width=\linewidth, title=Takeaway 1,  colbacktitle=blue!40,   %
coltitle=black,         %
fonttitle=\bfseries\itshape,
]
\begin{itemize}
    \item In the Exclusive + Homogeneous scenario, minimizing inference time is equivalent to minimizing communication time.
    \item \name determines the optimal token transmission order, ensuring that each GPU can receive data without bandwidth contention, thereby achieving minimum all-to-all communication time.
\end{itemize}
\end{tcolorbox}

\section{Exclusive Models on Heterogeneous Clusters}\label{sec:exclusive_hetero}

In this section, we derive the minimum inference time for running models exclusively on heterogeneous GPU clusters, referred to as the Exclusive + Heterogeneous scenario. 

\para{Solution overview.} We first show how to assign GPUs optimally ($\S\ref{sec:exclu_hetero_1}$). Next, we demonstrate that the transmission order obtained in homogeneous clusters remains optimal in a heterogeneous environment ($\S\ref{sec:exclu_hetero_2}$).

\subsection{Finding optimal GPU assignment}\label{sec:exclu_hetero_1}

\begin{wrapfigure}{r}{0.5\textwidth}
    \centering
    \includegraphics[width=1\linewidth]{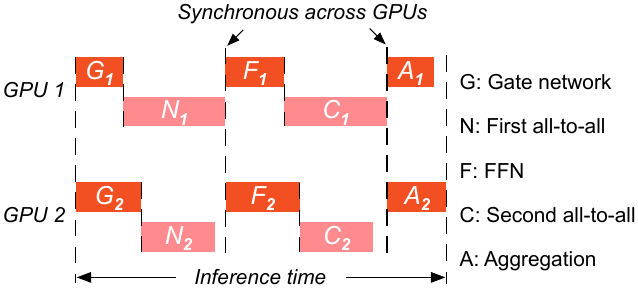}\vspace{-0.10in}
    \caption{Running \moe models exclusively on heterogeneous clusters.}\label{fig:single_hetero}
    \vspace{-0.25in}
\end{wrapfigure}

Fig.~\ref{fig:single_hetero} presents the exclusive + Heterogeneous scenario. Different from the homogeneous case, the three observations for Exclusive + Homogeneous in $\S$\ref{sec:exclusive_homo} do not hold. Most importantly, experts should be placed on the heterogeneous clusters carefully to reduce the inference time.

\begin{theorem}\label{thm:exclusive_hetero_assignment}
In a heterogeneous cluster, the optimal GPU assignment is to sort the experts by the number of tokens they process in descending order and then assign them to GPUs from the highest to the lowest performance.
\end{theorem}

\begin{proof}
Following Theorem~\ref{thm:exclusive_hetero_assignment}, we assign high-end GPUs to the most popular experts in descending order. Let's assume GPUs $m$ and $n$ are assigned to experts $p$ and $q$, respectively. GPU $m$ has higher performance than GPU $n$, and expert $p$ is more popular than expert $q$. The inference times on GPUs $m$ and $n$ are $t_m$ and $t_n$, respectively.

Now, suppose we reverse the assignment, mapping GPU $m$ to expert $q$ and GPU $n$ to expert $p$. The new inference times on GPUs $m$ and $n$ are $t'_m$ and $t'_n$, respectively. Because the more popular expert $p$ is now assigned to the lower-end GPU $n$, we have $t'_n > t_m$\footnote{In this work, a GPU with higher computational power will not have lower bandwidth compared to a lower-end GPU.}. Previously, GPU $n$ was handling the less popular expert $q$ with inference time $t_n$, so $t'_n > t_n$. Therefore, we can conclude that $t'_n > t_m$ and $t'_n > t_n$, indicating that we cannot achieve $\max(t'_m, t'_n) < \max(t_m, t_n)$. 

Thus, altering the assignment order outlined in Theorem~\ref{thm:exclusive_hetero_assignment} will not lead to a better solution.
\end{proof}

\subsection{Finding optimal transmission order}\label{sec:exclu_hetero_2}

Once we have determined the GPU assignment strategy, the computation time on each GPU is known. And the next step is to determine the communication time ($|N_i|$ and $|C_j|$ in Eqn.~\ref{eqn:single_homo_1}). In $\S$\ref{sec:exclusive_homo}, we propose Theorem~\ref{thm:exclusive_homo_comm_time} to calculate $\max(|N_i|)$ and $\max(|C_j|)$ in the homogeneous cluster. However, This theorem cannot be directly applied to the Exclusive + Heterogeneous scenario. The main difference is that network bandwidth varies across a heterogeneous cluster. Therefore, we propose an extension of Theorem~\ref{thm:exclusive_homo_comm_time} to address this issue.

\begin{theorem}\label{thm:exclusive_hetero_comm_time}
The transmission order obtained in homogeneous clusters remains optimal in a heterogeneous environment. The minimum communication time is $b_{max}$ = $max(\sum_{j=1}^{n}d_{ij}/ B_i,\; \sum_{i=1}^{n}d_{ij}/ B_i)$. 
\end{theorem}

Theorem~\ref{thm:exclusive_hetero_comm_time} states that the transmission order derived for homogeneous clusters remains optimal in a heterogeneous environment. The minimum communication time is determined by the GPU that takes the longest time to complete sending or receiving.

The proof for Theorem~\ref{thm:exclusive_hetero_comm_time} follows the same structure of Theorem~\ref{thm:exclusive_homo_comm_time}. The detailed proof can be found in Appx.~\ref{appendix:theorem2}.

\begin{tcolorbox}[colback=blue!3, colframe=blue!30!black, boxrule=0.25mm, width=\linewidth, title=Takeaway 2,  colbacktitle=blue!40,   %
coltitle=black,         %
fonttitle=\bfseries\itshape,
]
\begin{itemize}
    \item The GPU assignment affects the inference time in heterogeneous clusters.
    \item The optimal GPU assignment involves sorting experts by number of tokens processed, then assigning them to GPUs from highest to lowest performance.
    \item The transmission order developed for homogeneous clusters remains optimal in a heterogeneous environment.
\end{itemize}
\end{tcolorbox}

\section{Colocating Models on Homogeneous Clusters}\label{sec:colocating_homo}

In this section, we explore the best way to place two \moe models on a homogeneous cluster. This scenario is termed Colocating + Homogeneous.

\para{Solution overview.} We first demonstrate that the colocation choice affects the aggregated communication time and further the inference time. Next, we prove that a colocation solution minimizing aggregated communication time will also minimize inference time ($\S$\ref{sec:colo_homo_1}). We determine the optimal expert colocation, which minimizes communication time, by solving the bottleneck matching problem ($\S$\ref{sec:colo_homo_2}).

\subsection{Minimizing inference time equals minimizing communication time}\label{sec:colo_homo_1}

\begin{wrapfigure}{r}{0.48\textwidth}
    \centering
    \vspace{-0.05in}
    \includegraphics[width=1\linewidth]{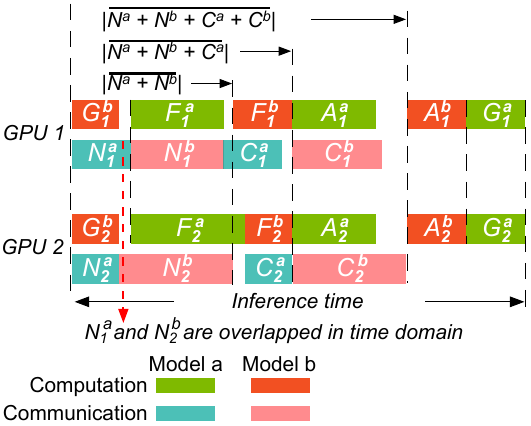}\vspace{-0.1in}
    \caption{Running colocating \moe models on homogeneous clusters.}\label{fig:moe_sharing_homo}
    \vspace{-0.2in}
\end{wrapfigure}

Fig.~\ref{fig:moe_sharing_homo} illustrates a scenario where two \moe models, $a$ and $b$, run simultaneously on a homogeneous cluster\footnote{We only colocate models with the same number of experts, even though it's not a strict requirement in theory.}. Components of Model $a$ are shown in shades of green, while those of Model $b$ are in shades of red. The subscript numbers (1 and 2) indicate the GPU index, and the superscript letters ($a$ and $b$) refer to the model index. For example, $G_1^b$ denotes the computation time of Model $b$'s Gate network on GPU 1.

The Colocating + Homogeneous scenario inherits characteristics of the Exclusive + Homogeneous case ($\S$\ref{sec:exclusive_homo}). These characteristics include: no need to decide GPU assignment on a homogeneous cluster, equal computation time for the Gate and Aggregation, and increased computation time for an expert when processing more tokens. Additionally, running two experts on the same GPU introduces new characteristics and constraints, as illustrated below.

\begin{enumerate}[label=(\arabic*), leftmargin=3em, topsep = 0.5em]
    \item \textit{Computation competition.} One model's computation components cannot start if another model is still under computing processes. 
    \item \textit{Communication overlapping.} The all-to-all communications from two models can overlap in the time domain. For instance, Model $a$'s all-to-all communication, $C_2^a$, can begin while Model $b$'s communication, $N_1^b$, on GPU 1 is still in progress.
\end{enumerate}

\para{Aggregated communication times.} The term $|\overline{N^a + N^b}|$ represents the time required to complete the first all-to-all communication for two models, which we refer to as the aggregated communication time. This differs from $|N^a| + |N^b|$, which simply adds the communication times of each model without considering potential overlap. As shown in Fig.~\ref{fig:moe_sharing_homo}, communications $N_1^a$ from Model $a$ and $N_2^b$ from Model $b$ overlap, resulting in $|\overline{N^a + N^b}|$ being smaller than $|N^a| + |N^b|$. $|\overline{N^a + N^b}|$ is impacted by the expert colocation choice. When colocating two experts, pairing one with high communication demands with another that has fewer tokens to send can reduce the aggregated communication time.

\para{Inference time expression.} Based on Fig.~\ref{fig:moe_sharing_homo}, we determine the finish time of each component, as shown in Table~\ref{tab:table1}. For simplicity, we display only the maximum start and end times across the $n$ GPUs for each component. For example, $|G^b|$ is defined as $\max(|G_i^b|)$ for $i \in [1, n]$, thereby omitting the GPU index subscript. Additionally, we use $E$ to denote the end time; for instance, $E_{A^b}$ indicates the end time of component $A^b$. The inference time corresponds to the end time of $G^a$, which is $E_{A^b} + |G^a|$, as shown below.

\vspace{-0.15in}
\begin{equation}\label{eqn:colocating_homo}
\textit{Inference time } \;\; t = E_{A^b} + |G^a|
\end{equation}
\vspace{-0.15in}

In Eqn.~\ref{eqn:colocating_homo}, $|G^a|$ is known in advance, while $E_{A^b}$, the end time of component $A^b$, is given by $\max(E_{A^a}, E_{C^b}) + |A^b|$. Both $E_{A^a}$ and $E_{C^b}$ can be further defined by the start and end times of other components. Following this approach, we can derive the complete expression for the inference time $t$, though it is not displayed here due to its complexity.

\begin{table}[tbp]
    \centering
    \footnotesize
    \caption{Start and end time of each component on the Colocating + Homogeneous scenario.}
    \vspace{-0.1in}
    \renewcommand{\arraystretch}{1.4} %
    \begin{tabular}{ccc}
        \toprule
        \textbf{Component} & \textbf{Start time} & \textbf{End time} \\
        \midrule
        $G^b$ & $0$ & $|G^b|$ \\       
        $N^a$ & $0$ & $|\overline{N^a}|$ \\   
        $F^a$ & $\max(|G^b|, |N^a|)$ & $\max(|G^b|, |N^a|) + |F^a|$ \\
        $N^b$ & $\geq |G^b|$ & $|\overline{N^a + N^b}|$ \\
        $F^b$ & $\max(E_{F^a}, E_{N^b})$ & $\max(E_{F^a}, E_{N^b}) + |F^b|$ \\
        $C^a$ & $\geq E_{F^a}$ & $|\overline{N^a + N^b + C^a}|$ \\
        $A^a$ & $\max(E_{F^b}, E_{C^a})$ & $\max(E_{F^b}, E_{C^a}) + |A^a|$ \\
        $C^b$ & $\geq E_{F^b}$ & $|\overline{N^a + N^b + C^a + C^b}|$ \\       
        $A^b$ & $\max(E_{A^a}, E_{C^b})$ & $\max(E_{A^a}, E_{C^b}) + |A^b|$ \\        
        $G^a$ & $E_{A^b}$ & $E_{A^b}$ + $|G^a|$ \\
        \bottomrule
    \end{tabular}
    \vspace{-0.1in}
    \label{tab:table1}
\end{table}

Rather than directly targeting inference time, we approach the problem by minimizing the aggregated communication time. In $\S$\ref{sec:exclusive_homo}, we minimize the communication time by determining an optimal transmission order using Theorem~\ref{thm:exclusive_homo_comm_time}. This is conducted under the context where traffic matrix is fixed. In contrast, colocation scenarios present different aggregated traffic matrices depending on the colocation choice. In the Colocating + Homogeneous scenario, minimizing communication time requires finding an expert colocation choice. The resulting traffic matrix achieves the shortest communication time when applying Theorem~\ref{thm:exclusive_homo_comm_time}.

\begin{theorem}\label{thm:colocating_homo_comm_time}
Minimizing aggregated all-to-all communication times of two colocating models ensures minimum inference time in a homogeneous cluster.
\end{theorem}

\begin{proof}
We use proof by contradiction. Assume we have an optimal colocating strategy that minimizes communication times, resulting in an inference time of $t = E_{A^b}$ + $|G^a|$. Now, suppose there exists another colocating strategy with higher communication times but a shorter inference time: $E'_{N^b} > E_{N^b}, E'_{C^a} > E_{C^a}, E'_{C^b} > E_{C^b}$, and $t' < t$. According to Theorem~\ref{thm:exclusive_homo_comm_time}, the minimum communication time is determined solely by the maximum column or row sum. Thus, different GPU assignment solutions for Model $a$ do not affect this value. So we have $E'_{N^a} = |N^a|' = E_{N^a} = |N^a|$. In a homogeneous cluster, we have $|G^a|' = |G^a|, |G^b|' = |G^b|, |A^a|' = |A^a|,$ and $|A^b|' = |A^b|$. Since computation time is proportional to communication time in such a cluster, it follows that $|F^a|' > |F^a|$ and $|F^b|' > |F^b|$. We will now proceed with the proof by contradiction to show it is impossible to achieve a lower inference time. Specifically, we need to prove that $t' = E'_{A^b} + |G^a|' < t = E_{A^b} + |G^a|$ cannot hold.

\vspace{-0.10in}
\begin{equation}\label{eqn:proof}
\begin{split}
&|G^a|' = |G^a|, \; E'_{A^b} + |G^a|' < E_{A^b} + |G^a| \Rightarrow \max(E'_{A^a}, E'_{C^b}) + |A^b|' < \max(E_{A^a}, E_{C^b}) + |A^b|\\
&|A^b|' = |A^b|, \; E'_{C^b} > E_{C^b} \Rightarrow E'_{A^a} < E_{A^a} \Rightarrow \max(E'_{F^b}, E'_{C^a}) + |A^a|' < \max(E_{F^b}, E_{C^a}) + |A^a|\\
&|A^a|' = |A^a|,\; E'_{C^a} >  E_{C^a} \Rightarrow E'_{F^b} < E_{F^b} \Rightarrow \max(E'_{F^a}, E'_{N^b}) + |F^b|' < \max(E_{F^a}, E_{N^b}) + |F^b|\\
& |F^b|' > |F^b|, \; E'_{N^b} > E_{N^b} \Rightarrow E'_{F^a} < E_{F^a} \Rightarrow \max(|G^b|', |N^a|') + |F^a|' < \max(|G^b|, |N^a|) + |F^a|\\
&|F^a|' > |F^a|, \; |N^a|' = |N^a| \Rightarrow |G^b|' < |G^b|
\end{split}
\end{equation}
\vspace{-0.10in}

Eqn.~\ref{eqn:proof} demonstrates that $|G^b|' < |G^b|$, which contradicts our earlier assertion that $|G^b|' = |G^b|$. This contradiction validates the correctness of Theorem~\ref{thm:colocating_homo_comm_time}.
\end{proof}

\subsection{Finding optimal expert colocation}\label{sec:colo_homo_2}

Next, we focus on finding an expert colocation which minimizes the communication time. In Table~\ref{tab:table1}, for the terms related to communication times directly, the priority is to reduce $|\overline{N^a + N^b}|$. Since $E_{N^a} = |N^a|$ is unaffected by the expert colocation solution, and $|\overline{C^a + C^b}|$ equals $|\overline{N^a + N^b}|$, minimizing $|\overline{N^a + N^b}|$ also minimizes $|\overline{N^a + N^b + C^a + C^b}|$. Because $N^a$ and $C^a$ do not overlap in the time domain (they are separated by $F^a$), $|\overline{N^a + N^b + C^a}|$ can be expressed as $|\overline{N^a + N^b}| + |\overline{C^a}|$. This highlights the importance of optimizing $|\overline{N^a + N^b}|$. 

In the following, we discuss how to find the expert colocation to minimize $|\overline{N^a + N^b}|$. Suppose $\mathbb{D}_{N_1}$ and $\mathbb{D}_{N_2}$ represent the first all-to-all communication traffic matrices of Model $a$ and Model $b$, respectively. For each potential expert colocating choice, by combining the traffic matrices of $\mathbb{D}_{N_1}$ and $\mathbb{D}_{N_2}$, we create a new traffic matrix $\mathbb{D}_{new}$. $\mathbb{D}_{new}$ represents the aggregated traffic matrix of the two colocating models. According to Theorem~\ref{thm:exclusive_homo_comm_time}, the communication time with $\mathbb{D}_{new}$ is determined solely by the maximum sum of traffic within each column or row. In the following, we try to \textit{find the optimal expert colocating solution, which minimizes the maximum sum of each column or row among all possible $\mathbb{D}_{new}$}.

For traffic matrix $\mathbb{D}_{N_1}$, the sending and receiving traffic on GPU $i$ are denoted as $a_{i} =\sum_{j=1}^{n}d_{ij}$ and $a_{n+i} = \sum_{j=1}^{n}d_{ji}$, respectively. Utilizing a vector $\mathbf{a}$ = $[(a_{1}$, $a_{n+1})$, $(a_{2}$, $a_{n+2})$ ,..., $(a_{n}$, $a_{2n})]$, the sending/receiving traffic on $n$ GPUs is represented. Similarly, for $\mathbb{D}_{N_2}$, a vector $\mathbf{b}$ = $[(b_{1}$, $b_{n+1})$, $(b_{2}$, $b_{n+2})$ ,..., $(b_{n}$, $b_{2n})]$ is generated. For an expert colocating choice, an element is selected from both $\mathbf{a}$ and $\mathbf{b}$ each time, creating a new vector $\mathbf{h}$, which is shown in the following equation\footnote{Considering the synchronous all-to-all communication constraint in Eqn.~\ref{eqn:h}, we have $\mathbf{h} = \mathbf{a}_{i} + \mathbf{b}_{j} = (a_{i} + \max(b_{j}, G^b), a_{n+i} + \max(b_{n+j}, G^b))$. However, this adjustment does not impact the optimal expert colocating choice. For clarity in the following proof, we will continue to use Eqn.~\ref{eqn:h}.}. 

\begin{equation}\label{eqn:h}
\mathbf{h} = \mathbf{a}_{i} + \mathbf{b}_{j} = (a_{i} + b_{j}, a_{n+i} + b_{n+j}), \; i, j \in [1, n]
\end{equation}

The vector $\mathbf{h}$ represents the column/row sum of the traffic matrix $\mathbb{D}_{new}$. Next, we try to minimize the maximum value in $\mathbf{h}$.

\begin{wrapfigure}{r}{0.62\textwidth}
    \centering
    \vspace{-0.1in}
    \includegraphics[width=1\linewidth]{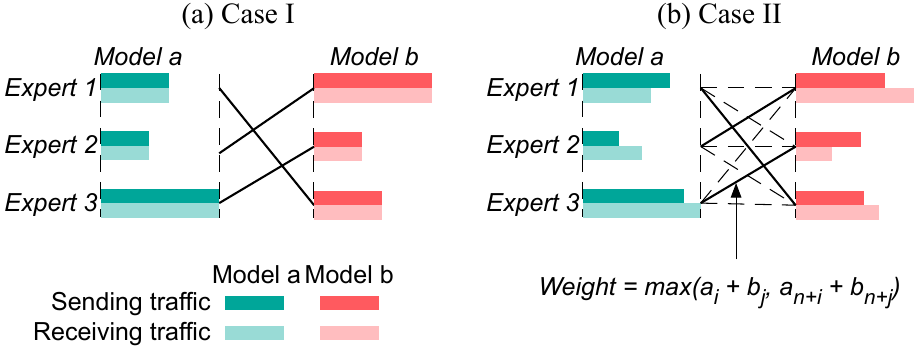}\vspace{-0.1in}
    \caption{(a) Case I: The optimal expert colocation solution is alternating between selecting one popular and one unpopular expert from Model $a$ and Model $b$. (b) Case II: Solving the \bomp yields the optimal expert colocation solution.}\label{fig:mwm_sharing_homo}
    \vspace{-0.2in}
\end{wrapfigure}

We minimize the maximum value in $\mathbf{h}$ under two cases. In Case I (Fig.~\ref{fig:mwm_sharing_homo}(a)), the quantity of sending traffic is equal to the receiving traffic for each GPU. In Case II (Fig.~\ref{fig:mwm_sharing_homo}(b)), the sending traffic may not necessarily be equal to the receiving traffic. Case I can be considered a specific instance of Case II. The reason for categorizing this problem into two cases is the availability of a lower-complexity algorithm tailored for Case I.

\indent \para{\textbf{Case I: The amount of sending traffic is equal to the receiving traffic for each GPU}}

For this particular case, we employ the following theorem to determine the optimal expert colocation solution.

\begin{theorem}\label{thm:a+b}
Vectors $\mathbf{a}$ and $\mathbf{b}$ are of the same sizes. Vector $\mathbf{h}$ is formed by adding the values selected each from $\mathbf{a}$ and $\mathbf{b}$. Sort $\mathbf{a}$ in ascending and $\mathbf{b}$ in descending order. Selecting values from $\mathbf{a}$ and $\mathbf{b}$ sequentially minimizes the maximum value in $\mathbf{h}$.
\end{theorem}

\begin{proof}
\begin{equation}\label{eqn:a+b}
\left(\begin{NiceArray}{ccc|c|ccc}
a_{1} & \cdots & a_{k-1} & a_{k} & a_{k+1} & \cdots & a_{n}\\ 
b_{n} & \cdots & b_{n-k+2} & b_{n-k+1} & b_{n-k} & \cdots &b_{1}
\end{NiceArray}\right)
\end{equation}

Assuming $\mathbf{a}$ and $\mathbf{b}$ are sorted in ascending and descending order, respectively, denoted as $a_{1} \leq a_{2} \leq...\leq a_{n}$ and $b_{n} \geq b_{n-1} \geq...\geq b_{1}$, as depicted in Eqn.~\ref{eqn:a+b}. By summing the elements from $\mathbf{a}$ and $\mathbf{b}$ in a sequential manner, we obtain $\mathbf{h} = [a_{1}+b_{n}, a_{2}+b_{n-1}, ..., a_{n}+b_{1}]$. Without loss of generality, let's assume the $k_{th}$ item of $\mathbf{h}$, which is $a_{k}+b_{n-k+1}$, serves as the maximum value. Our next objective is to demonstrate that it is impossible to rearrange $\mathbf{a}$ and $\mathbf{b}$ to create a new vector $\mathbf{h}$, where the maximum value is smaller than $a_{k}+b_{n-k+1}$.

Let's consider the subarray $\mathbf{a}_{[k+1, n]} = [a_{k+1}, a_{k+2}, ..., a_{n}] \subseteq \mathbf{a}$. For any value $a_{i} \in $ $\mathbf{a}_{[k+1, n]}$,  it is apparent that $a_{i} \geq a_{k}$. To maintain the maximum value of $\mathbf{h}$ below $a_{k}+b_{n-k+1}$, $a_{i}$ must be paired with a value smaller than $b_{n-k+1}$. For any value $b_{i} \in \mathbf{b}_{[1, n-k]} = [b_{1}, b_{2}, ..., b_{n-k}]$, it holds that $b_{i} \leq b_{n-k+1}$. Therefore, the values from $\mathbf{a}_{[k+1, n]}$ must be matched with the values from $\mathbf{b}_{[1, n-k]}$. Similarly, the subarray $\mathbf{a}_{[1, k-1]} = [a_{1}, a_{2}, ..., a_{k-1}] \subseteq \mathbf{a}$  must be matched with the values from $\mathbf{b}_{[n-k+2, n]} = [b_{n-k+2}, b_{n-k+3}, ..., b_{n}] \subseteq \mathbf{b}$. Now the elements from $\mathbf{a}_{[k+1, n]}$ are paired with elements from $\mathbf{b}_{[1, n-k]}$, and $\mathbf{a}_{[1, k-1]}$ is paired with $\mathbf{b}_{[n-k+2, n]}$. With only $a_{k}$ and $b_{n-k+1}$ remaining, they must be paired together, obtaining $a_{k}+b_{n-k+1}$. In this scenario, the maximum value of the new vector cannot be less than $a_{k}+b_{n-k+1}$.
\end{proof}

The idea behind Theorem~\ref{thm:a+b} is to alternate between selecting one large and one small value from $\mathbf{a}$ and $\mathbf{b}$. This means we should colocate a popular expert from Model $a$ and an unpopular expert from Model $b$. This strategy reduces the aggregated communication time on that GPU where these two experts are colocated.

\indent \para{\textbf{Case II: The amount of sending traffic is not equal to the receiving traffic}}

In Case II, where $a_{i} \neq a_{n+i}$ and $b_{j} \neq b_{n+j}$, Theorem~\ref{thm:a+b} is not applicable. This is because we cannot sort $\mathbf{a}$ and $\mathbf{b}$ where one element contains two distinct values. Attempting to sort $\mathbf{a}$ and $\mathbf{b}$ according to the larger value inside an element and then applying Theorem~\ref{thm:a+b} does not minimize the maximum value in $\mathbf{h}$.

We reformulate the problem as a matching problem. In Fig.~\ref{fig:mwm_sharing_homo}(b), we construct the graph as follows: the experts of Model $a$ and Model $b$ are represented as nodes on the left and right sides of the bipartite graph, respectively. Each node on the left is connected to every node on the right with an edge weighted by $max(a_{i} + b_{j}$,  $a_{n+i} + b_{n+j})$, where $i, j \in [1, n]$. This creates a fully connected bipartite graph with $n^2$ edges. The weight of each edge indicates the maximum amount of data transmitted (sent or received) by a GPU if the corresponding experts from each model are colocated.

There is a direct one-to-one correspondence between the mappings of sequences $\mathbf{a}$ and $\mathbf{b}$ and perfect matchings in the constructed bipartite graph. A perfect matching is one that covers all nodes. Finding an optimal sequence mapping is thus equivalent to identifying a perfect matching that minimizes the maximum edge weight. This problem is known as the bottleneck matching problem~\cite{bottleneck}. The algorithm for solving the bottleneck matching problem is straightforward. It involves a binary search~\cite{knuth1998art} on the sorted array of edges to find the minimum weight $w_{min}$ such that a perfect matching exists in the subgraph induced by all edges with weights not exceeding $w_{min}$. The existence of a perfect matching in a bipartite graph can be verified using the Hopcroft-Karp algorithm~\cite{hop1973kraft}, which has a complexity of $O(n^{2}\sqrt n)$. Combined with binary search, the overall complexity is $O(n^{2}\sqrt nlogn)$.

With the optimal expert colocation solution determined, we can derive the communication time and the corresponding optimal inference time as outlined in Table~\ref{tab:table1}.

In summary, this section outlines the optimal expert colocation solution to minimize inference time on homogeneous clusters. We demonstrate that optimizing the communication time for colocated models is crucial. Using Theorem~\ref{thm:a+b} and solving the bottleneck matching problem, we identify the expert colocation solution that achieves the minimum inference time.

\begin{tcolorbox}[colback=blue!3, colframe=blue!30!black, boxrule=0.25mm, width=\linewidth, title=Takeaway 3,  colbacktitle=blue!40,   %
coltitle=black,         %
fonttitle=\bfseries\itshape,
]
\begin{itemize}
    \item Expert colocation choices impact both the aggregated communication time and, consequently, the inference time.
    \item Minimizing aggregated communication time ensures minimum inference time in a homogeneous cluster.
    \item \name identifies the optimal expert colocation by solving the bottleneck matching problem, thus achieving the minimum inference time.
\end{itemize}
\end{tcolorbox}

\section{Colocating Models on Heterogeneous Clusters}\label{sec:colocating_hetero}

In this section, we focus on colocating models on heterogeneous clusters. Achieving minimum inference time in the Colocating + Heterogeneous scenario requires expert colocation, GPU assignment, and communication scheduling.

\para{Solution overview.} We first identify that optimizing inference time in the Colocating + Heterogeneous scenario is an NP-hard problem ($\S$\ref{sec:colo_hetero_1}), and then we propose a sub-optimal yet effective solution ($\S$\ref{sec:colo_hetero_2}).

\subsection{NP-hardness proof}\label{sec:colo_hetero_1}

Fig.~\ref{fig:sharing_hetero_a} illustrates the case of running two \moe models on a heterogeneous cluster. Similar to the Colocating + Homogeneous scenario, the inference time can be expressed using Eqn.~\ref{eqn:colocating_homo}, with the finish times for each component detailed in Table~\ref{tab:table1}.

\begin{wrapfigure}{r}{0.44\textwidth}
    \centering
    \vspace{-0.15in}
    \includegraphics[width=1\linewidth]{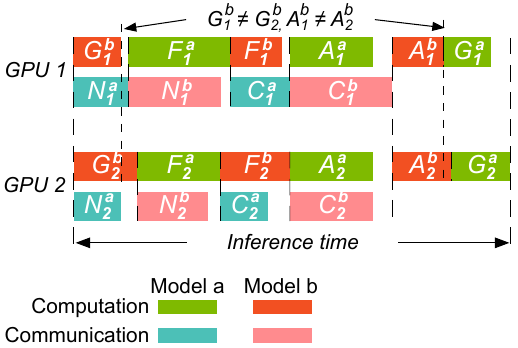}\vspace{-0.1in}
    \caption{Running colocating \moe models on heterogeneous clusters.}\label{fig:sharing_hetero_a}
    \vspace{-0.10in}
\end{wrapfigure}

In $\S$\ref{sec:colocating_homo}, Theorem~\ref{thm:colocating_homo_comm_time} demonstrates that minimizing aggregated communication times ensures the minimum inference time on a homogeneous cluster. However, this theorem does not apply to a heterogeneous cluster. In the homogeneous environment, computation times are identical across GPUs. Therefore, we have $|G^a|' = |G^a|, |G^b|' = |G^b|, |A^a|' = |A^a|,$ and $|A^b|' = |A^b|$ in the proof of Theorem~\ref{thm:colocating_homo_comm_time}. Additionally, we apply $|F^a|' > |F^a|$ and $|F^b|' > |F^b|$, indicating that computation time is proportional to communication time. However, these equations and inequalities do not hold in a heterogeneous cluster. As demonstrated in Fig.~\ref{fig:sharing_hetero_a}, $G_1^b \neq G_2^b$, $A_1^b \neq A_2^b$, rendering Theorem~\ref{thm:colocating_homo_comm_time} inapplicable in such heterogeneous environments.

We can reformulate the optimization problem as a 3-dimensional matching problem, as illustrated in Fig.~\ref{fig:sharing_hetero_b}(a). Unlike the scenario depicted in Fig.~\ref{fig:mwm_sharing_homo}, this formulation requires both expert colocation and GPU assignment. The 3-dimensional matching problem extends bipartite matching (also known as 2-dimensional matching). A hyperedge, connecting one GPU and one expert from Model $a$ and one expert from Model $b$, represents the inference time occurring on that GPU. We must determine two perfect matchings among two bipartite graphs. Similar to the bottleneck matching problem applied in Case II ($\S$\ref{sec:colocating_homo}), we need to find a perfect matching that minimizes the maximum weight. The 3-dimensional matching problem is proven to be NP-hard~\cite{3DM}, meaning that we cannot solve the optimization problem in polynomial time.

\subsection{Sub-optimal approach}\label{sec:colo_hetero_2}

We use a sub-optimal solution by decoupling the matchings in the two bipartite graphs.

\begin{figure}[tb]
    \centering
    \includegraphics[width=1\linewidth]{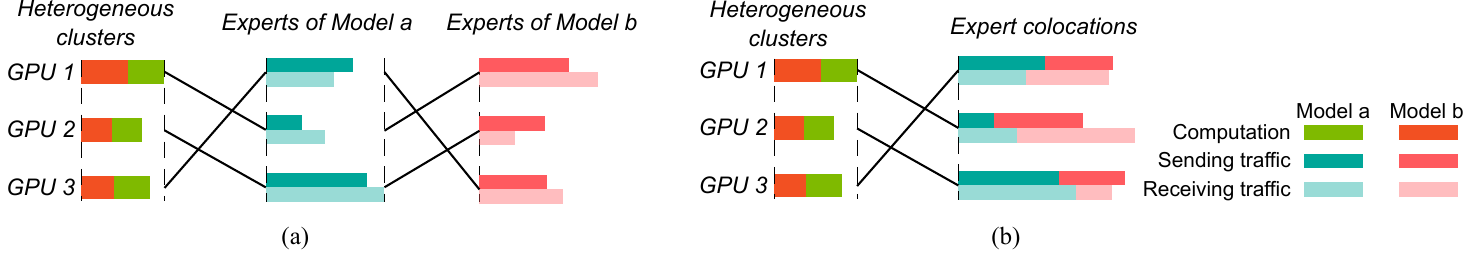}%
    \caption{(a) Optimal expert colocation and GPU assignment solution is obtained by solving a 3-dimensional matching problem. (b) We can reduce the 3-dimensional matching problem to two 2-dimensional matching problems.}\label{fig:sharing_hetero_b}
    \vspace{-0.2in}
\end{figure}

We first determine the perfect matching among experts, setting aside GPU assignment initially. Following the method described in Case II ($\S$\ref{sec:colocating_homo}), we solve the bottleneck matching problem to obtain the expert colocation solution. This reduces the 3-dimensional matching problem to a 2-dimensional matching problem. In Fig.~\ref{fig:sharing_hetero_b}(b), the left side represents GPUs, and the right side represents the combination of two experts, with the edge weight indicating inference time on the connected GPU. We resolve the bottleneck matching problem to determine the minimum of the maximum weights. Combined with the expert colocation solution, this provides a complete, sub-optimal solution.

In conclusion, achieving minimum inference time in the Colocating + Heterogeneous scenario can be formulated as a 3-dimensional matching problem, which is proven to be NP-hard~\cite{3DM}. Based on our evaluation in $\S$\ref{sec:evaluation}, this solution achieves an inference time just 1.07$\times$ of the optimal.

\begin{tcolorbox}[colback=blue!3, colframe=blue!30!black, boxrule=0.25mm, width=\linewidth, title=Takeaway 4,  colbacktitle=blue!40,   %
coltitle=black,         %
fonttitle=\bfseries\itshape,
]
\begin{itemize}
    \item In a  heterogeneous cluster, minimizing aggregated all-to-all communication times of two colocating models does not ensure minimum inference time.
    \item Minimizing inference time in the Colocating + Heterogeneous scenario can be formulated as an NP-hard matching problem.
    \item We propose a sub-optimal approach by decoupling the optimization problem into two perfect matching problems.
\end{itemize}
\end{tcolorbox}

\section{Evaluation}\label{sec:evaluation}

The evaluation seeks to address the following key questions.

\para{Q1: Can \name reduce inference time across four scenarios?} \name achieves up to 1.38$\times$ faster inference time in the Exclusive + Homogeneous scenario and up to 1.81$\times$ faster in the Exclusive + Heterogeneous scenario. In the colocating scenario, \name shows an improvement of up to 2.38$\times$ in the homogeneous case and up to 3.54$\times$ in the heterogeneous case.

\para{Q2: Can \name improve GPU utilization?} In the colocation scenario, \name delivers a 1.28$\times$ to 1.50$\times$ improvement in GPU utilization compared to the state-of-the-art solution.

\para{Q3: How close is \name to the optimum in the Colocating + Heterogeneous scenario?} On average, \name prolongs the inference time by only 1.07$\times$ compared to the optimum.

\para{Q4: How does \name perform under imprecise traffic inputs?} \name maintains inference time performance under unpredictable inference requests, with only a 15.8\% degradation.

\subsection{Simulation setup}\label{sec:setup}

\para{GPU clusters.} The GPUs are connected through a large switch, as shown in Fig.~\ref{fig:network}(a). In homogeneous clusters, the network bandwidth is set to $\uGbps{100}$. For heterogeneous clusters, we define four types of GPUs, with bandwidths of $\uGbps{100}$, $\uGbps{80}$, $\uGbps{50}$, and $\uGbps{40}$, ordered from highest to lowest performance. The number of GPUs for each type is the same. In the exclusive scenario, each \moe model uses the network bandwidth independently. In the colocation scenario, models only compete bandwidth when their experts are placed on the same device.

\para{\moe models.}
We use production model statistics from Google~\cite{limoe} to drive our simulation. It includes data for four layers of two \moe models, B/16 and B/32, each with 8 experts. We derive \name's input parameters from the model information based on the COCO and ImageNet datasets.

\para{Metrics.} We consider the following metrics in the evaluation.

\begin{itemize}
\renewcommand{\labelitemi}{\scriptsize$\bullet$}
    \item \textit{Inference time}. We calculate the inference time for all four scenarios.
    \item \textit{GPU utilization}. GPU utilization is the ratio of computation time (including the Gate, FFN, and Aggregation) to the inference time.
\end{itemize}

\para{Baselines.} \name is the first of its kind, making it difficult to find directly comparable work. For expert colocation, we compare \name with Lina~\cite{li2023accelerating}, the latest approach using expert colocation. We also implement vanilla expert colocation, referred to as random expert colocation (REC), as the baseline. To ensure fairness, all solutions colocate two experts on the same device. Lina\footnote{Lina consists of three main components: prioritizing all-to-all over all-reduce, pipelining communication and computation, and packing multiple experts on a single device. The first component is specific to \moe training and does not apply to \name. The second complements \name, while the third is closely related. We implement the third component for Lina.} pairs the most popular expert with the least popular one within each job, while \name and REC colocate experts from two different models.

For GPU assignment in heterogeneous clusters, we use the vanilla approach, random GPU assignment (RGA), as the baseline.

For all-to-all communication scheduling, we employ the shortest job first (SJF), which is a well-known flow scheduling policy for minimizing average flow completion time. We also include the vanilla method, random communication scheduling (RCS).

\subsection{Results}\label{sec:results}

\begin{figure*}[t]
    \centering
    \begin{minipage}[t]{0.495\linewidth}
        \centering
        \includegraphics[width=\linewidth]{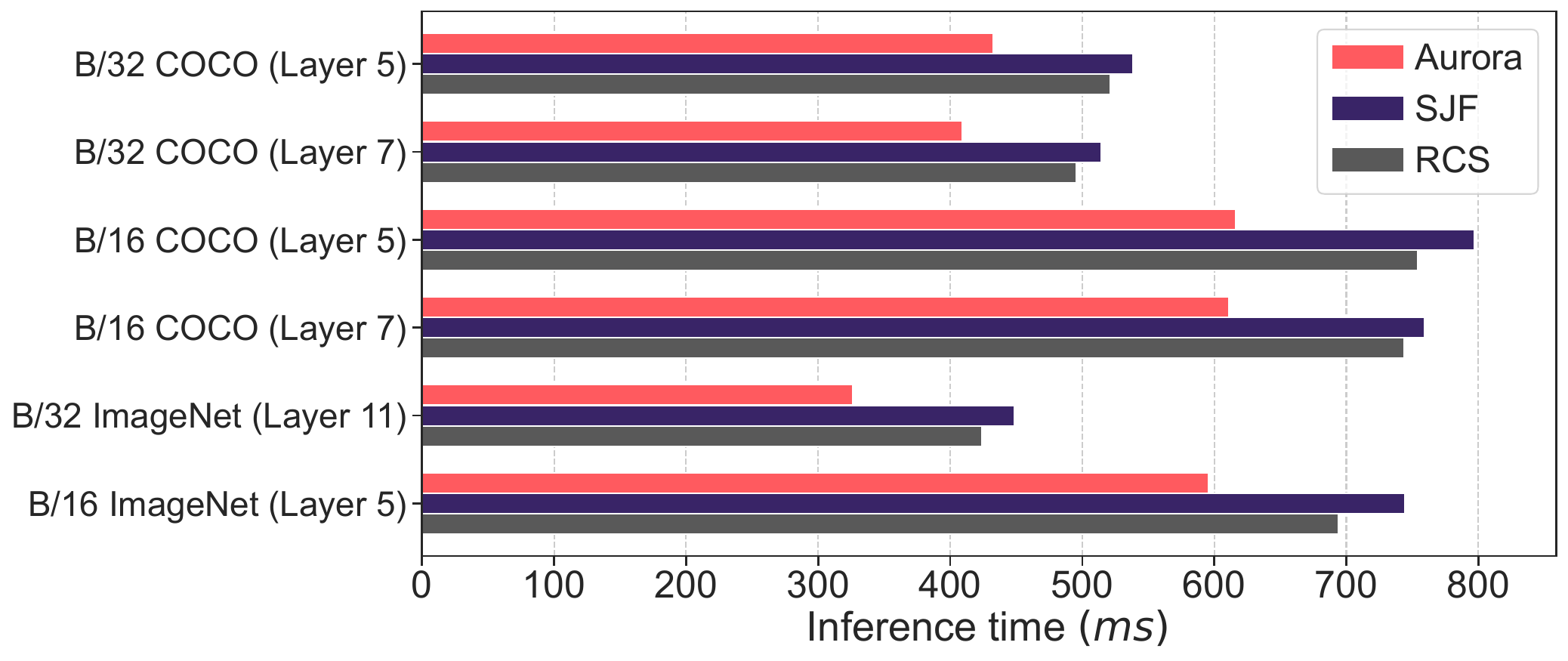}\vspace{-0.05in}
        \subcaption{}\label{fig:infer_time_case_1}
    \end{minipage}%
    \hfill
    \begin{minipage}[t]{0.495\linewidth}
        \centering
        \includegraphics[width=\linewidth]{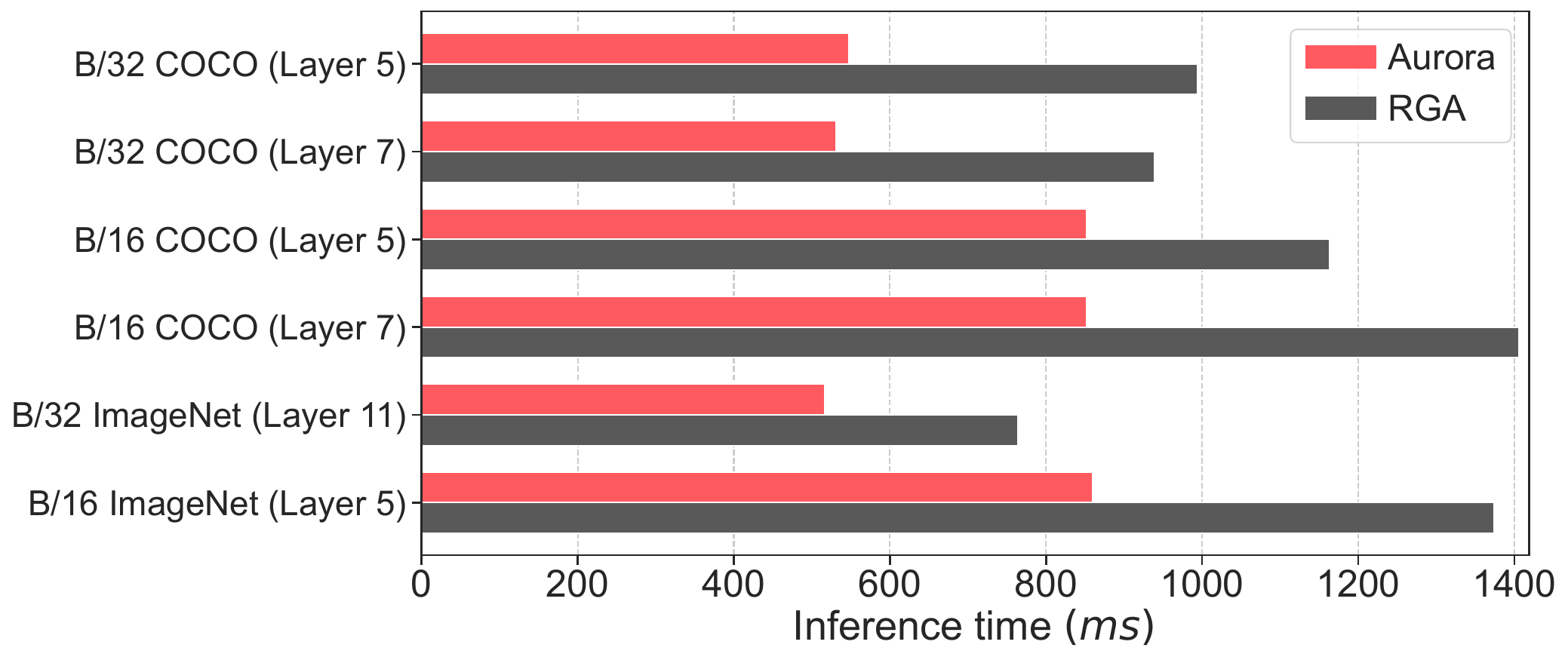}\vspace{-0.05in}
        \subcaption{}\label{fig:infer_time_case_2}
    \end{minipage}%
    \vspace{0.15cm} %
    \begin{minipage}[t]{0.495\linewidth}
        \centering
        \includegraphics[width=\linewidth]{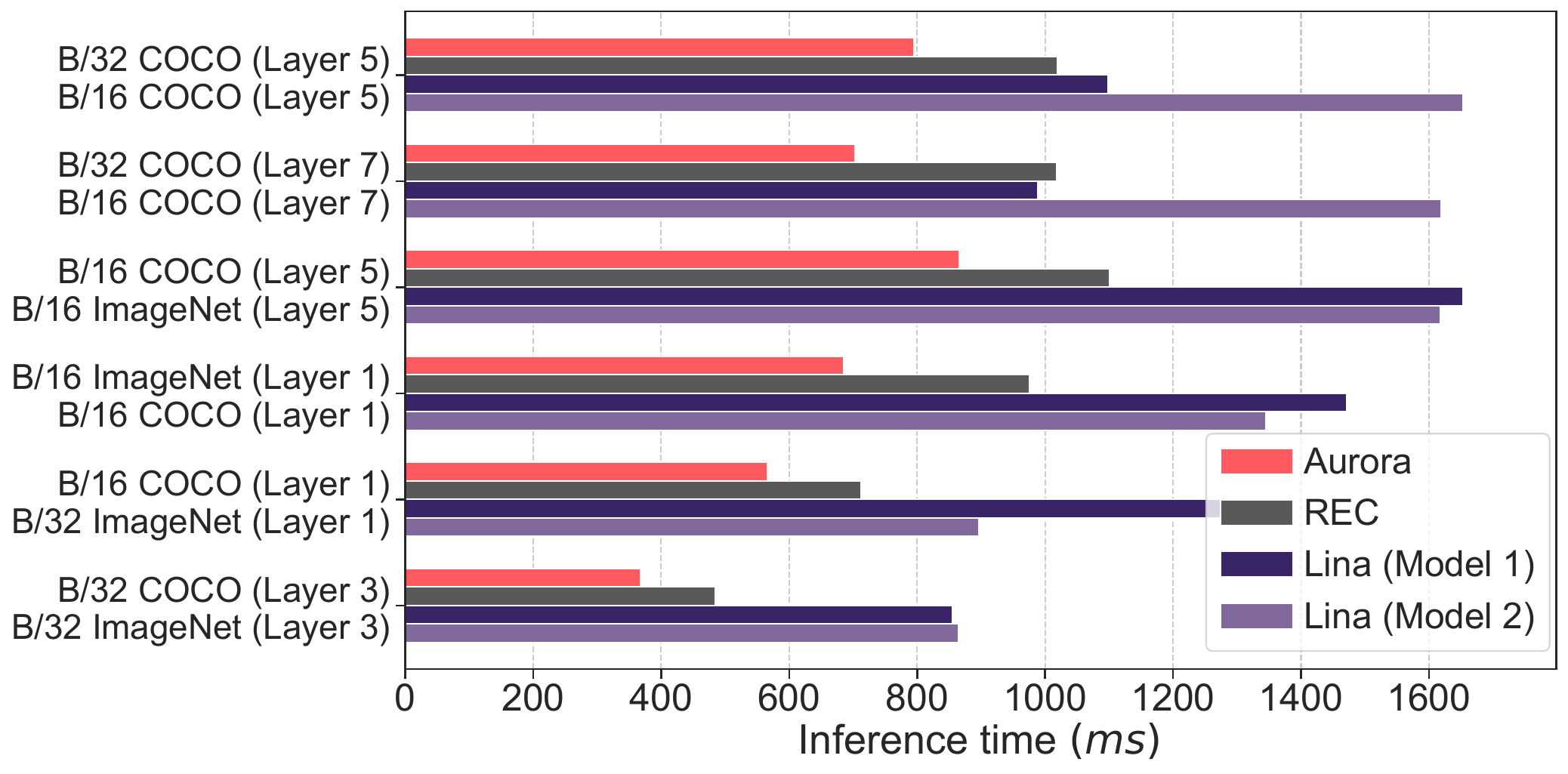}\vspace{-0.05in}
        \subcaption{}\label{fig:infer_time_case_3}
    \end{minipage}%
    \hfill
    \begin{minipage}[t]{0.495\linewidth}
        \centering
        \includegraphics[width=\linewidth]{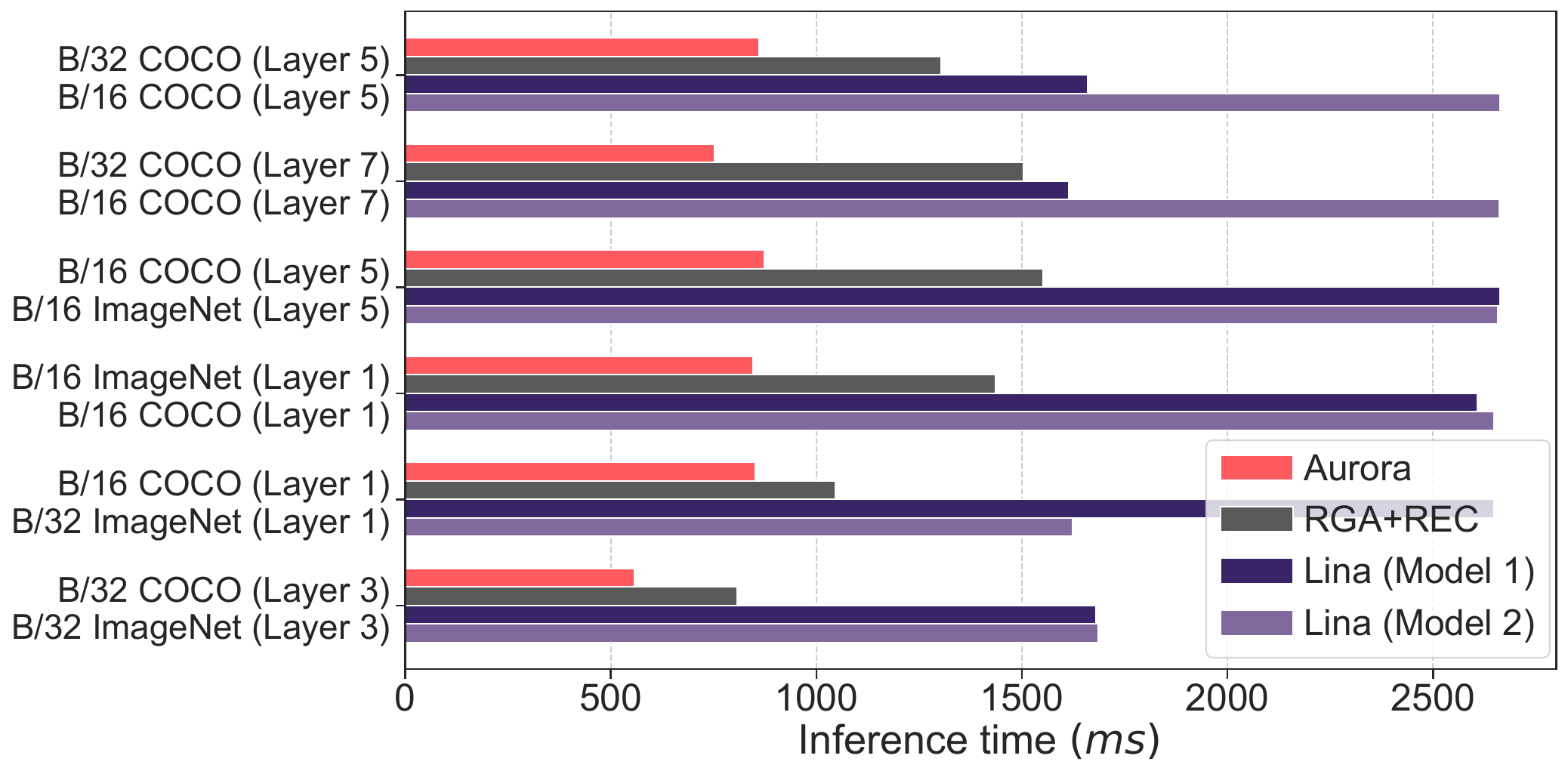}\vspace{-0.05in}
        \subcaption{}\label{fig:infer_time_case_4}
    \end{minipage}%
    \vspace{-0.15in}
    \caption{Inference time comparison in (a) Exclusive + Homogeneous, (b) Exclusive + Heterogeneous, (c) Colocating + Homogeneous, and (d) Colocating + Heterogeneous scenarios.}
    \label{fig:infer_time}
    \vspace{-0.15in}
\end{figure*}

\begin{figure*}[t]
    \begin{minipage}[t]{1\linewidth}
    \centering
        \begin{minipage}[t]{0.495\linewidth}
        \includegraphics[width=\linewidth]
        {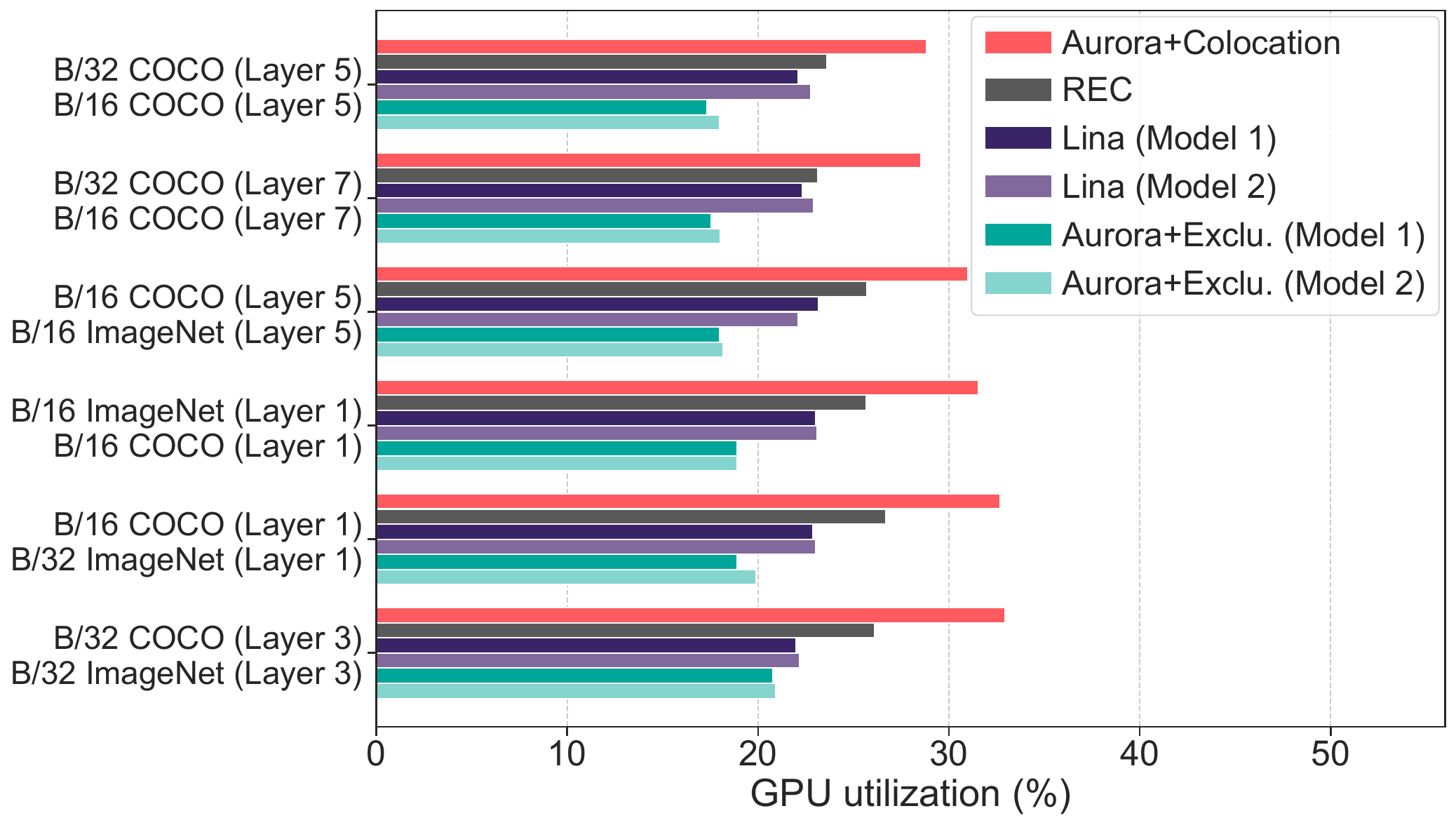}\vspace{-0.10in}
        \subcaption{}\label{fig:gpu_util_case3}
        \vspace{-0.1in}
        \end{minipage}%
        \hfill
        \begin{minipage}[t]{0.495\linewidth}
        \includegraphics[width=\linewidth]
        {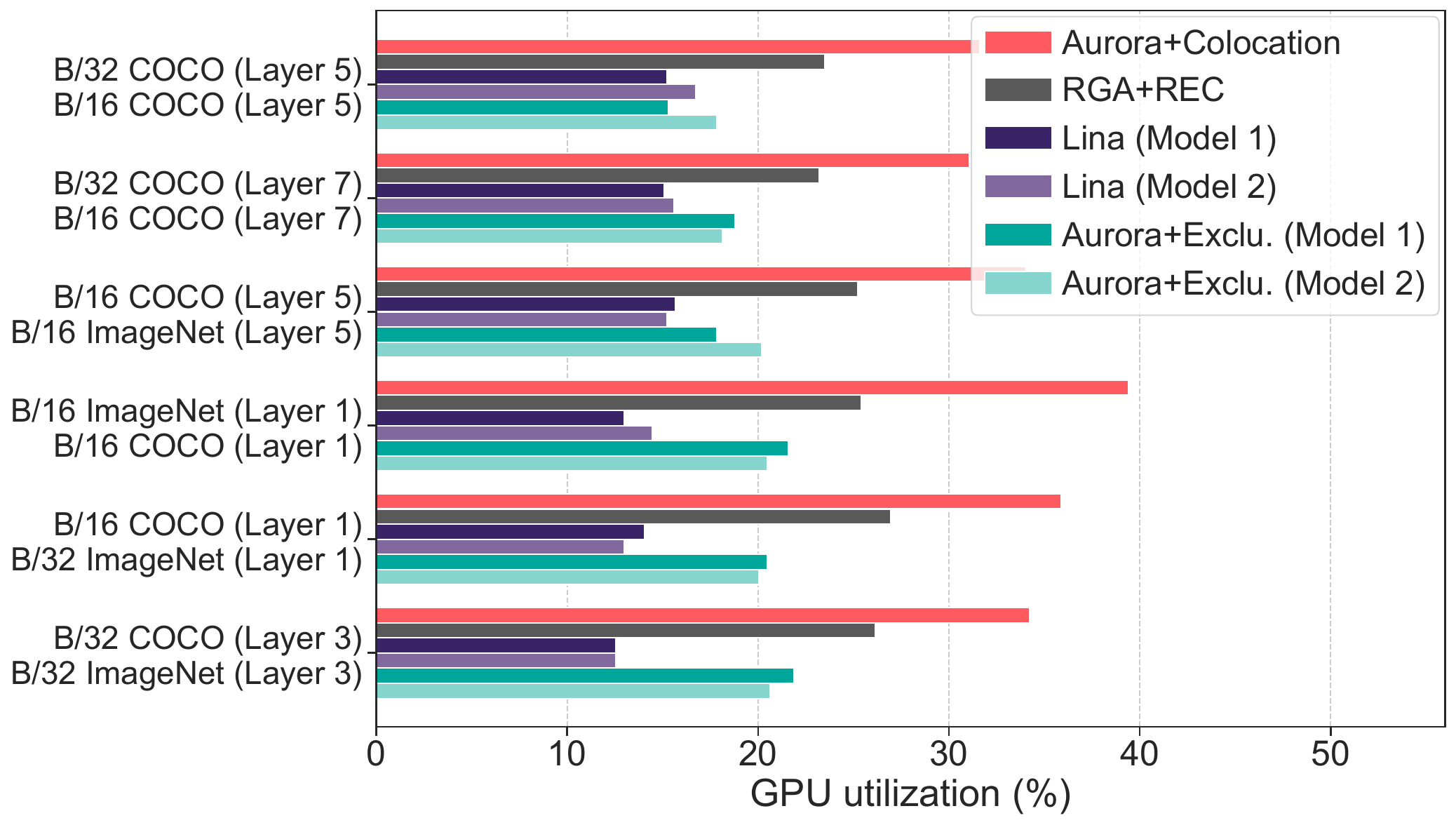}\vspace{-0.10in}
        \subcaption{}\label{fig:gpu_util_case4}
        \vspace{-0.1in}
        \end{minipage}%
    \caption{GPU utilization in the (a) Colocating + Homogeneous and (b) Colocating + Heterogeneous scenarios.}\label{fig:gpu_util}
    \end{minipage}
    \vspace{-0.20in}
\end{figure*}

\para{(Q1) \name reduces inference time across four scenarios.} We evaluate inference time across various scenarios. Fig.~\ref{fig:infer_time_case_1} shows a comparison of inference times for three scheduling algorithms: \name, SJF, and RCS. These algorithms decide the order of token transmission between GPUs during all-to-all communication, based on the traffic matrix of each layer. \name consistently outperforms both SJF and RCS across all model layers and datasets, achieving communication times that are up to 1.38$\times$ faster compared to SJF. This demonstrates its efficiency in minimizing communication time. In contrast, SJF shows performance similar to or even worse than RCS. This is because prioritizing tokens with less traffic offers no advantage. SJF's inability to reduce bandwidth contention results in outcomes that are nearly identical to those of RCS.

In Fig.~\ref{fig:infer_time_case_2}, we present a comparison of inference times between \name and RGA in the Exclusive + Heterogeneous scenario. With \name, inference times are accelerated by 1.36$\times$ to 1.81$\times$ across various models and layers. This is achieved by assigning popular experts to high-end GPUs, optimizing overall inference performance.

Furthermore, Fig.~\ref{fig:infer_time_case_3} illustrates the inference time when two experts are colocated on the same GPU. Lina colocates experts from the same model, and we show the inference time for each model separately. \name consistently achieves the shortest inference time compared to Lina, REC, and RGA + REC. Under the homogeneous case, \name is 1.25$\times$ to 2.38$\times$ faster than Lina, while in the heterogeneous scenario (Fig.~\ref{fig:infer_time_case_4}), it improves by 1.91$\times$ to 3.54$\times$. \name places experts from two different models, allowing them to avoid the synchronous all-to-all communication constraint. In contrast, with Lina, colocated experts must wait for each other to complete communication, which can lead to longer inference times.

\para{(Q2) \name improves GPU utilization.} Fig.~\ref{fig:gpu_util_case3} illustrates GPU utilization in the homogeneous case. \name + Colocation refers to placing two experts on the same GPU, while \name + Exclusive represents assigning one expert per GPU. GPU utilization is notably low when running \moe models exclusively, with most models below 20\%. By colocating two experts on a single device, \name achieves a 1.57$\times$ to 1.72$\times$ increase in GPU utilization. In general, colocating two experts is expected to nearly double GPU utilization, but the observed improvement is lower. This is because, with multiple experts sharing a GPU, inference time for \name + Colocation is longer compared to \name + Exclusive, reducing the potential GPU utilization gains. \name achieves a significant improvement over Lina, with an increase of 1.28$\times$ to 1.50$\times$. A similar trend is observed in the heterogeneous case, as shown in Fig.~\ref{fig:gpu_util_case4}.

\begin{wrapfigure}{r}{0.55\textwidth}
    \centering
    \vspace{-0.15in}
    \includegraphics[width=1\linewidth]{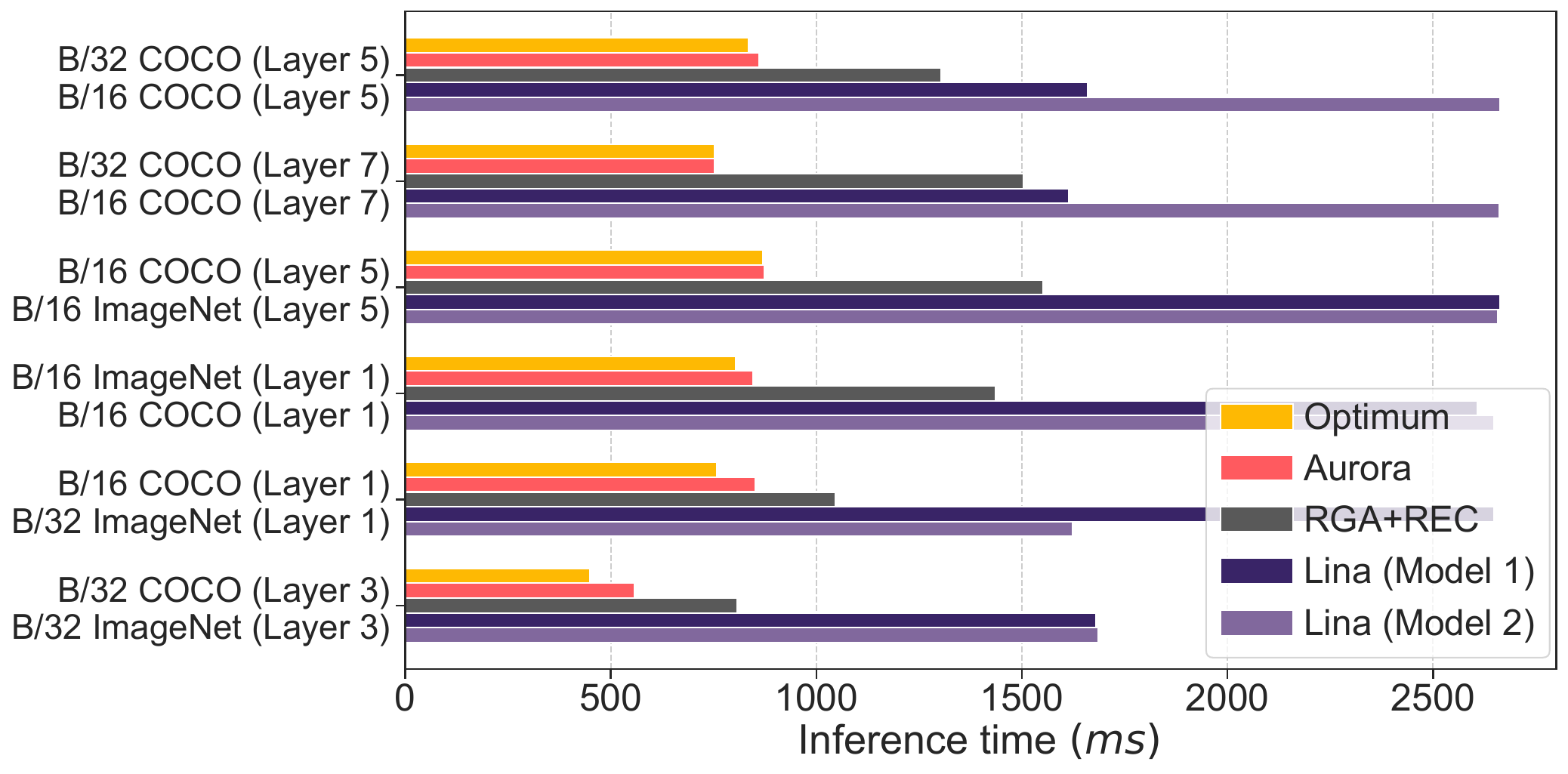}\vspace{-0.15in}
    \caption{Performance gap between \name and the optimum in the Colocating + Heterogeneous scenario.}\label{fig:gap_optimal}\vspace{-0.15in}
\end{wrapfigure}

\para{(Q3) \name realizes close performance to the optimal solution.} \name achieves minimal inference time in most scenarios, except for the Colocating + Heterogeneous case. Fig.~\ref{fig:gap_optimal} shows the inference time gap between \name and the optimum, obtained through brute-force search. On average, \name prolongs the inference time by only 1.07$\times$, which is a small difference given that it significantly outperforms other baseline methods.

\begin{figure*}[t]
    \begin{minipage}[t]{1\linewidth}
    \centering
        \begin{minipage}[t]{0.49\linewidth}
        \includegraphics[width=\linewidth]
        {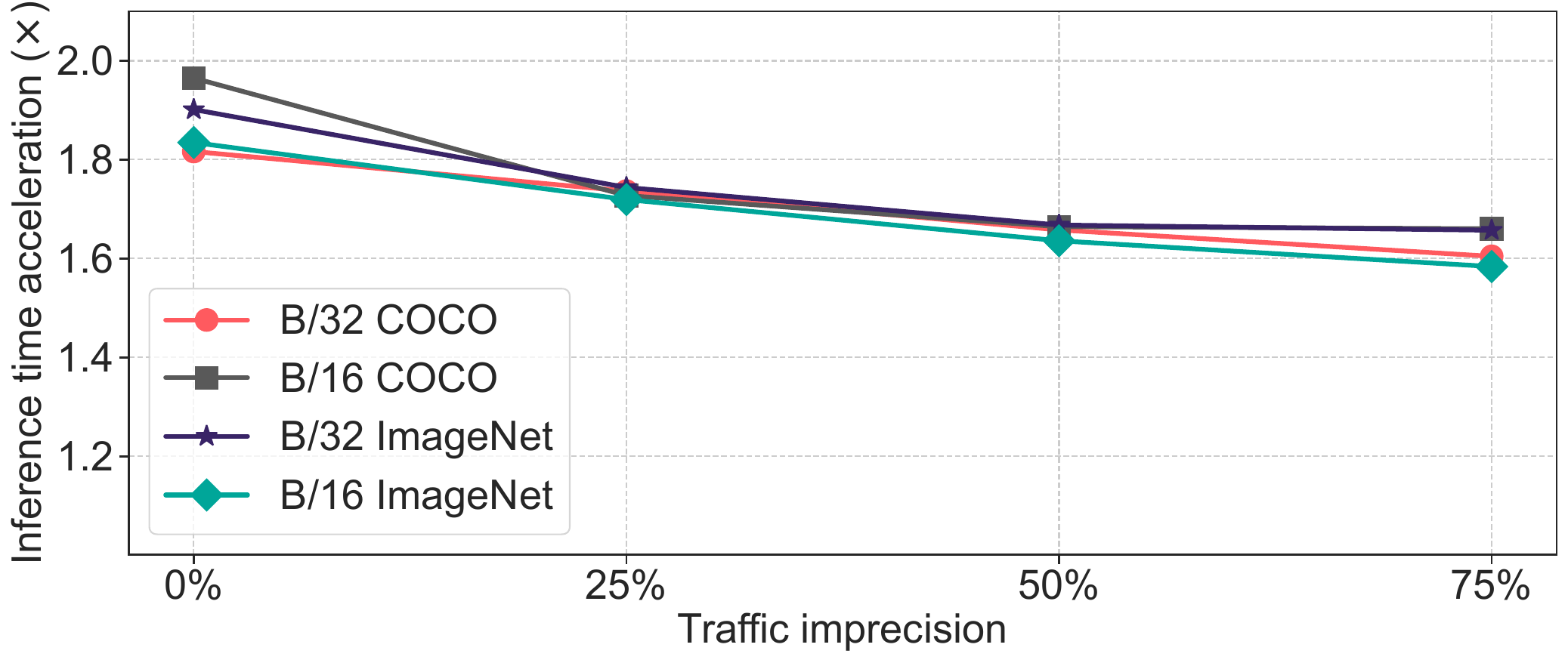}\vspace{-0.10in}
        \subcaption{}\label{fig:infer_mlayer_case2}
        \vspace{-0.15in}
        \end{minipage}%
        \hfill
        \begin{minipage}[t]{0.49\linewidth}
        \includegraphics[width=\linewidth]
        {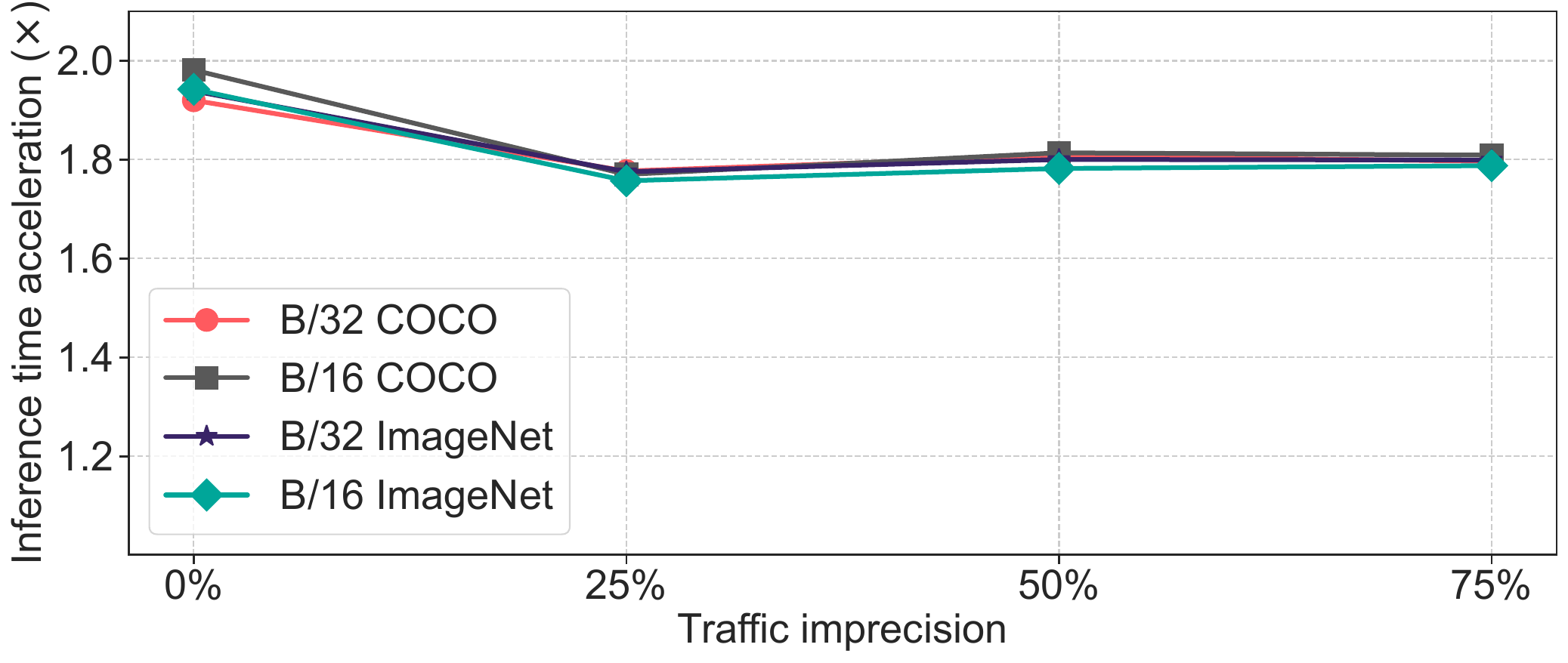}\vspace{-0.10in}
        \subcaption{}\label{fig:infer_mlayer_case4}
        \vspace{-0.15in}
        \end{minipage}%
    \caption{Inference time acceleration across different number of layers in the (a) Exclusive + Heterogeneous and (b) Colocating + Heterogeneous scenarios.}\label{fig:infer_mlayer}
    \end{minipage}
    \vspace{-0.3in}
\end{figure*}

\para{(Q4) \name maintains inference time performance under imprecise traffic inputs.}
Once \name's optimization plan is deployed, subsequent inference requests become unpredictable. We use the traffic matrix of the first layer for \name optimization and add traffic from the other three layers in the dataset as noise to simulate unpredictable requests. The level of imprecision ranges from 0\% to 75\% as traffic from each additional layer is incorporated.

Fig.~\ref{fig:infer_mlayer} shows the inference time acceleration of \name compared to RGA (Fig.~\ref{fig:infer_mlayer_case2}) and RGA+REC (Fig.~\ref{fig:infer_mlayer_case4}) across varying traffic imprecision. As expected, the inference time reduction generally decreases slightly as the traffic matrix becomes more imprecise. As expected, inference time reduction decreases slightly with increased traffic matrix imprecision. In the Exclusive+Heterogeneous scenario, acceleration drops from approximately 1.90$\times$ with precise traffic to 1.60$\times$ with 75\% imprecision. Similarly, in the Colocating+Heterogeneous scenario, acceleration decreases from about 2.0$\times$ to 1.80$\times$. The maximum performance degradation is 15.8\% with high noise traffic, demonstrating that \name still achieves significant inference time improvements even with imprecise inputs.

\vspace{-0.15in}
\section{Related Work}\label{sec:related_work}

\para{Load balancing.} Various gating methods have been proposed to ensure even token distribution~\cite{fedus2022switch, lepikhin2020gshard, riquelme2021scaling, he2022fastermoe, chen2022ta}. Some use an auxiliary loss function to penalize imbalances~\cite{shazeer2017outrageously, hwang2023tutel}, while others regulate expert capacity~\cite{child2019generating, hwang2023tutel}. Dynamic-Gating \moe~\cite{huang2023towards} allows experts to process a variable number of tokens, and Pre-Gated \moe~\cite{hwang2023pre} predicts token distribution based on the previous layer's gate. However, these methods can impede model convergence and degrade overall quality.

\para{All-to-all acceleration.} All-to-all communication is a key bottleneck in MoE model training and inference~\cite{rajbhandari2022deepspeed, he2022fastermoe}. To improve efficiency, Faster-MoE~\cite{he2022fastermoe} uses a pairwise exchange algorithm, Tutel~\cite{hwang2023tutel} introduces hierarchical strategies, and DeepSpeed-MoE~\cite{rajbhandari2022deepspeed} employs tensor parallelism and slicing. Fast-MoE~\cite{he2021fastmoe} also uses tensor slicing and data parallelism. However, these methods mainly focus on communication speed, neglecting GPU utilization and system heterogeneity.

\para{Expert colocation and replication.} Some solutions improve GPU utilization by colocating multiple experts from the same job. Lina~\cite{li2023accelerating} packs experts on a single GPU to reduce all-to-all transfer sizes, while Dynamic-Gating \moe~\cite{huang2023towards} offloads less-used experts to CPU memory. Other methods replicate popular experts across GPUs. FlexMoE~\cite{Flexmoe} dynamically shifts experts based on workload, and Prophet~\cite{Prophet} maps experts to specific GPU subsets. Lazarus~\cite{Lazarus} uses expert deployment and replication to enhance training during GPU failures. However, these approaches do not address copacking experts from different models or solving all-to-all communication challenges.

\para{GPU heterogeneity.} GPU heterogeneity is becoming more common in production clusters~\cite{mlaas} and has drawn significant academic interest~\cite{gavel, sia, Hap, heterog}. While these solutions improve the management of heterogeneous GPU clusters, they do not specifically address \moe models.

\para{Flow scheduling.} Some existing research~\cite{echelonflow, BlueConnect, ATP, PLink, MLfabric} attempts to model various training paradigms and optimize flow scheduling to speed up the process. However, these approaches often overlook the specific advantages of all-to-all communication.

\para{GPU sharing.} Experts colocated on the same device share GPU resources, and GPU sharing has been extensively explored in previous research~\cite{Gandiva, Salus, Antman, Wavelet}. Even though these solutions do not explore \moe models, the engineering techniques they employ can help reduce overhead when multiple experts share GPU resources alternately.

\vspace{-0.1in}
\section{Conclusion}\label{sec:conclusion}

In conclusion, \name effectively addresses key challenges in \moe inference by optimizing model deployment and communication scheduling. While this work marks an important first step, it opens up several promising avenues for future research. One direction is extending \name to handle more complex environments, including those with varying network topologies and communication protocols. Another potential enhancement involves developing adaptive strategies that dynamically adjust model deployment and communication scheduling based on changing workloads, which could further improve performance. Additionally, integrating \name with other optimization techniques, such as job scheduling and network topology design, may provide further synergistic benefits. These efforts aim to improve the scalability and efficiency of \moe models in increasingly diverse and demanding computing environments.

\clearpage

\balance

\bibliographystyle{ACM-Reference-Format}
\bibliography{moe}


\begin{thebibliography}{44}


\ifx \showCODEN    \undefined \def \showCODEN     #1{\unskip}     \fi
\ifx \showDOI      \undefined \def \showDOI       #1{#1}\fi
\ifx \showISBNx    \undefined \def \showISBNx     #1{\unskip}     \fi
\ifx \showISBNxiii \undefined \def \showISBNxiii  #1{\unskip}     \fi
\ifx \showISSN     \undefined \def \showISSN      #1{\unskip}     \fi
\ifx \showLCCN     \undefined \def \showLCCN      #1{\unskip}     \fi
\ifx \shownote     \undefined \def \shownote      #1{#1}          \fi
\ifx \showarticletitle \undefined \def \showarticletitle #1{#1}   \fi
\ifx \showURL      \undefined \def \showURL       {\relax}        \fi
\providecommand\bibfield[2]{#2}
\providecommand\bibinfo[2]{#2}
\providecommand\natexlab[1]{#1}
\providecommand\showeprint[2][]{arXiv:#2}

\bibitem[far(2024)]%
        {farkas}
 \bibinfo{year}{2024}\natexlab{}.
\newblock \bibinfo{title}{Farkas' Lemma}.
\newblock \bibinfo{howpublished}{\url{https://en.wikipedia.org/wiki/Farkas\%27_lemma}}.
\newblock


\bibitem[Burkard and Derigs(1980)]%
        {bottleneck}
\bibfield{author}{\bibinfo{person}{Rainer~E Burkard} {and} \bibinfo{person}{Ulrich Derigs}.} \bibinfo{year}{1980}\natexlab{}.
\newblock \showarticletitle{The bottleneck matching problem}.
\newblock In \bibinfo{booktitle}{\emph{Assignment and Matching Problems: Solution Methods with FORTRAN-Programs}}. \bibinfo{publisher}{Springer}, \bibinfo{pages}{60--71}.
\newblock


\bibitem[Chen et~al\mbox{.}(2022)]%
        {chen2022ta}
\bibfield{author}{\bibinfo{person}{Chang Chen}, \bibinfo{person}{Min Li}, \bibinfo{person}{Zhihua Wu}, \bibinfo{person}{Dianhai Yu}, {and} \bibinfo{person}{Chao Yang}.} \bibinfo{year}{2022}\natexlab{}.
\newblock \showarticletitle{Ta-moe: Topology-aware large scale mixture-of-expert training}.
\newblock \bibinfo{journal}{\emph{Advances in Neural Information Processing Systems}}  \bibinfo{volume}{35} (\bibinfo{year}{2022}), \bibinfo{pages}{22173--22186}.
\newblock


\bibitem[Child et~al\mbox{.}(2019)]%
        {child2019generating}
\bibfield{author}{\bibinfo{person}{Rewon Child}, \bibinfo{person}{Scott Gray}, \bibinfo{person}{Alec Radford}, {and} \bibinfo{person}{Ilya Sutskever}.} \bibinfo{year}{2019}\natexlab{}.
\newblock \showarticletitle{Generating long sequences with sparse transformers}.
\newblock \bibinfo{journal}{\emph{arXiv preprint arXiv:1904.10509}} (\bibinfo{year}{2019}).
\newblock


\bibitem[Cho et~al\mbox{.}(2019)]%
        {BlueConnect}
\bibfield{author}{\bibinfo{person}{Minsik Cho}, \bibinfo{person}{Ulrich Finkler}, \bibinfo{person}{David Kung}, {and} \bibinfo{person}{Hillery Hunter}.} \bibinfo{year}{2019}\natexlab{}.
\newblock \showarticletitle{Blueconnect: Decomposing all-reduce for deep learning on heterogeneous network hierarchy}.
\newblock \bibinfo{journal}{\emph{Proceedings of Machine Learning and Systems}}  \bibinfo{volume}{1} (\bibinfo{year}{2019}), \bibinfo{pages}{241--251}.
\newblock


\bibitem[Crama and Spieksma(1992)]%
        {3DM}
\bibfield{author}{\bibinfo{person}{Yves Crama} {and} \bibinfo{person}{Frits~CR Spieksma}.} \bibinfo{year}{1992}\natexlab{}.
\newblock \showarticletitle{Approximation algorithms for three-dimensional assignment problems with triangle inequalities}.
\newblock \bibinfo{journal}{\emph{European Journal of Operational Research}} \bibinfo{volume}{60}, \bibinfo{number}{3} (\bibinfo{year}{1992}), \bibinfo{pages}{273--279}.
\newblock


\bibitem[Fedus et~al\mbox{.}(2022)]%
        {fedus2022switch}
\bibfield{author}{\bibinfo{person}{William Fedus}, \bibinfo{person}{Barret Zoph}, {and} \bibinfo{person}{Noam Shazeer}.} \bibinfo{year}{2022}\natexlab{}.
\newblock \showarticletitle{Switch transformers: Scaling to trillion parameter models with simple and efficient sparsity}.
\newblock \bibinfo{journal}{\emph{Journal of Machine Learning Research}} \bibinfo{volume}{23}, \bibinfo{number}{120} (\bibinfo{year}{2022}), \bibinfo{pages}{1--39}.
\newblock


\bibitem[He et~al\mbox{.}(2021)]%
        {he2021fastmoe}
\bibfield{author}{\bibinfo{person}{Jiaao He}, \bibinfo{person}{Jiezhong Qiu}, \bibinfo{person}{Aohan Zeng}, \bibinfo{person}{Zhilin Yang}, \bibinfo{person}{Jidong Zhai}, {and} \bibinfo{person}{Jie Tang}.} \bibinfo{year}{2021}\natexlab{}.
\newblock \showarticletitle{Fastmoe: A fast mixture-of-expert training system}.
\newblock \bibinfo{journal}{\emph{arXiv preprint arXiv:2103.13262}} (\bibinfo{year}{2021}).
\newblock


\bibitem[He et~al\mbox{.}(2022)]%
        {he2022fastermoe}
\bibfield{author}{\bibinfo{person}{Jiaao He}, \bibinfo{person}{Jidong Zhai}, \bibinfo{person}{Tiago Antunes}, \bibinfo{person}{Haojie Wang}, \bibinfo{person}{Fuwen Luo}, \bibinfo{person}{Shangfeng Shi}, {and} \bibinfo{person}{Qin Li}.} \bibinfo{year}{2022}\natexlab{}.
\newblock \showarticletitle{Fastermoe: modeling and optimizing training of large-scale dynamic pre-trained models}. In \bibinfo{booktitle}{\emph{Proceedings of the 27th ACM SIGPLAN Symposium on Principles and Practice of Parallel Programming}}. \bibinfo{pages}{120--134}.
\newblock


\bibitem[Hopcroft and Karp(1973)]%
        {hop1973kraft}
\bibfield{author}{\bibinfo{person}{John~E. Hopcroft} {and} \bibinfo{person}{Richard~M. Karp}.} \bibinfo{year}{1973}\natexlab{}.
\newblock \showarticletitle{An $n^{5/2}$ Algorithm for Maximum Matchings in Bipartite Graphs}.
\newblock \bibinfo{journal}{\emph{SIAM J. Comput.}} \bibinfo{volume}{2}, \bibinfo{number}{4} (\bibinfo{year}{1973}), \bibinfo{pages}{225--231}.
\newblock


\bibitem[Huang et~al\mbox{.}(2023)]%
        {huang2023towards}
\bibfield{author}{\bibinfo{person}{Haiyang Huang}, \bibinfo{person}{Newsha Ardalani}, \bibinfo{person}{Anna Sun}, \bibinfo{person}{Liu Ke}, \bibinfo{person}{Hsien-Hsin~S Lee}, \bibinfo{person}{Anjali Sridhar}, \bibinfo{person}{Shruti Bhosale}, \bibinfo{person}{Carole-Jean Wu}, {and} \bibinfo{person}{Benjamin Lee}.} \bibinfo{year}{2023}\natexlab{}.
\newblock \showarticletitle{Towards MoE Deployment: Mitigating Inefficiencies in Mixture-of-Expert (MoE) Inference}.
\newblock \bibinfo{journal}{\emph{arXiv preprint arXiv:2303.06182}} (\bibinfo{year}{2023}).
\newblock


\bibitem[Hwang et~al\mbox{.}(2023a)]%
        {hwang2023tutel}
\bibfield{author}{\bibinfo{person}{Changho Hwang}, \bibinfo{person}{Wei Cui}, \bibinfo{person}{Yifan Xiong}, \bibinfo{person}{Ziyue Yang}, \bibinfo{person}{Ze Liu}, \bibinfo{person}{Han Hu}, \bibinfo{person}{Zilong Wang}, \bibinfo{person}{Rafael Salas}, \bibinfo{person}{Jithin Jose}, \bibinfo{person}{Prabhat Ram}, {et~al\mbox{.}}} \bibinfo{year}{2023}\natexlab{a}.
\newblock \showarticletitle{Tutel: Adaptive mixture-of-experts at scale}.
\newblock \bibinfo{journal}{\emph{Proceedings of Machine Learning and Systems}}  \bibinfo{volume}{5} (\bibinfo{year}{2023}).
\newblock


\bibitem[Hwang et~al\mbox{.}(2023b)]%
        {hwang2023pre}
\bibfield{author}{\bibinfo{person}{Ranggi Hwang}, \bibinfo{person}{Jianyu Wei}, \bibinfo{person}{Shijie Cao}, \bibinfo{person}{Changho Hwang}, \bibinfo{person}{Xiaohu Tang}, \bibinfo{person}{Ting Cao}, \bibinfo{person}{Mao Yang}, {and} \bibinfo{person}{Minsoo Rhu}.} \bibinfo{year}{2023}\natexlab{b}.
\newblock \showarticletitle{Pre-gated moe: An algorithm-system co-design for fast and scalable mixture-of-expert inference}.
\newblock \bibinfo{journal}{\emph{arXiv preprint arXiv:2308.12066}} (\bibinfo{year}{2023}).
\newblock


\bibitem[Jayaram~Subramanya et~al\mbox{.}(2023)]%
        {sia}
\bibfield{author}{\bibinfo{person}{Suhas Jayaram~Subramanya}, \bibinfo{person}{Daiyaan Arfeen}, \bibinfo{person}{Shouxu Lin}, \bibinfo{person}{Aurick Qiao}, \bibinfo{person}{Zhihao Jia}, {and} \bibinfo{person}{Gregory~R Ganger}.} \bibinfo{year}{2023}\natexlab{}.
\newblock \showarticletitle{Sia: Heterogeneity-aware, goodput-optimized ML-cluster scheduling}. In \bibinfo{booktitle}{\emph{Proceedings of the 29th Symposium on Operating Systems Principles}}. \bibinfo{pages}{642--657}.
\newblock


\bibitem[Knuth(1998)]%
        {knuth1998art}
\bibfield{author}{\bibinfo{person}{D.E. Knuth}.} \bibinfo{year}{1998}\natexlab{}.
\newblock \bibinfo{booktitle}{\emph{The Art of Computer Programming: Sorting and Searching, Volume 3}}.
\newblock \bibinfo{publisher}{Pearson Education}.
\newblock
\showISBNx{9780321635785}
\urldef\tempurl%
\url{https://books.google.de/books?id=cYULBAAAQBAJ}
\showURL{%
\tempurl}


\bibitem[Lao et~al\mbox{.}(2021)]%
        {ATP}
\bibfield{author}{\bibinfo{person}{ChonLam Lao}, \bibinfo{person}{Yanfang Le}, \bibinfo{person}{Kshiteej Mahajan}, \bibinfo{person}{Yixi Chen}, \bibinfo{person}{Wenfei Wu}, \bibinfo{person}{Aditya Akella}, {and} \bibinfo{person}{Michael Swift}.} \bibinfo{year}{2021}\natexlab{}.
\newblock \showarticletitle{{ATP}: In-network Aggregation for Multi-tenant Learning}. In \bibinfo{booktitle}{\emph{18th USENIX Symposium on Networked Systems Design and Implementation (NSDI 21)}}. \bibinfo{pages}{741--761}.
\newblock


\bibitem[Lepikhin et~al\mbox{.}(2020)]%
        {lepikhin2020gshard}
\bibfield{author}{\bibinfo{person}{Dmitry Lepikhin}, \bibinfo{person}{HyoukJoong Lee}, \bibinfo{person}{Yuanzhong Xu}, \bibinfo{person}{Dehao Chen}, \bibinfo{person}{Orhan Firat}, \bibinfo{person}{Yanping Huang}, \bibinfo{person}{Maxim Krikun}, \bibinfo{person}{Noam Shazeer}, {and} \bibinfo{person}{Zhifeng Chen}.} \bibinfo{year}{2020}\natexlab{}.
\newblock \showarticletitle{Gshard: Scaling giant models with conditional computation and automatic sharding}.
\newblock \bibinfo{journal}{\emph{arXiv preprint arXiv:2006.16668}} (\bibinfo{year}{2020}).
\newblock


\bibitem[Li et~al\mbox{.}(2023)]%
        {li2023accelerating}
\bibfield{author}{\bibinfo{person}{Jiamin Li}, \bibinfo{person}{Yimin Jiang}, \bibinfo{person}{Yibo Zhu}, \bibinfo{person}{Cong Wang}, {and} \bibinfo{person}{Hong Xu}.} \bibinfo{year}{2023}\natexlab{}.
\newblock \showarticletitle{Accelerating distributed $\{$MoE$\}$ training and inference with lina}. In \bibinfo{booktitle}{\emph{2023 USENIX Annual Technical Conference (USENIX ATC 23)}}. \bibinfo{pages}{945--959}.
\newblock


\bibitem[Liu et~al\mbox{.}(2023)]%
        {Janus}
\bibfield{author}{\bibinfo{person}{Juncai Liu}, \bibinfo{person}{Jessie~Hui Wang}, {and} \bibinfo{person}{Yimin Jiang}.} \bibinfo{year}{2023}\natexlab{}.
\newblock \showarticletitle{Janus: A unified distributed training framework for sparse mixture-of-experts models}. In \bibinfo{booktitle}{\emph{Proceedings of the ACM SIGCOMM 2023 Conference}}. \bibinfo{pages}{486--498}.
\newblock


\bibitem[Luo et~al\mbox{.}(2020)]%
        {PLink}
\bibfield{author}{\bibinfo{person}{Liang Luo}, \bibinfo{person}{Peter West}, \bibinfo{person}{Jacob Nelson}, \bibinfo{person}{Arvind Krishnamurthy}, {and} \bibinfo{person}{Luis Ceze}.} \bibinfo{year}{2020}\natexlab{}.
\newblock \showarticletitle{Plink: Discovering and exploiting locality for accelerated distributed training on the public cloud}.
\newblock \bibinfo{journal}{\emph{Proceedings of Machine Learning and Systems}}  \bibinfo{volume}{2} (\bibinfo{year}{2020}), \bibinfo{pages}{82--97}.
\newblock


\bibitem[Mustafa et~al\mbox{.}(2022)]%
        {limoe}
\bibfield{author}{\bibinfo{person}{Basil Mustafa}, \bibinfo{person}{Carlos Riquelme}, \bibinfo{person}{Joan Puigcerver}, \bibinfo{person}{Rodolphe Jenatton}, {and} \bibinfo{person}{Neil Houlsby}.} \bibinfo{year}{2022}\natexlab{}.
\newblock \showarticletitle{Multimodal contrastive learning with limoe: the language-image mixture of experts}.
\newblock \bibinfo{journal}{\emph{Advances in Neural Information Processing Systems}}  \bibinfo{volume}{35} (\bibinfo{year}{2022}), \bibinfo{pages}{9564--9576}.
\newblock


\bibitem[Narayanan et~al\mbox{.}(2020)]%
        {gavel}
\bibfield{author}{\bibinfo{person}{Deepak Narayanan}, \bibinfo{person}{Keshav Santhanam}, \bibinfo{person}{Fiodar Kazhamiaka}, \bibinfo{person}{Amar Phanishayee}, {and} \bibinfo{person}{Matei Zaharia}.} \bibinfo{year}{2020}\natexlab{}.
\newblock \showarticletitle{$\{$Heterogeneity-Aware$\}$ cluster scheduling policies for deep learning workloads}. In \bibinfo{booktitle}{\emph{14th USENIX Symposium on Operating Systems Design and Implementation (OSDI 20)}}. \bibinfo{pages}{481--498}.
\newblock


\bibitem[Nie et~al\mbox{.}(2023)]%
        {Flexmoe}
\bibfield{author}{\bibinfo{person}{Xiaonan Nie}, \bibinfo{person}{Xupeng Miao}, \bibinfo{person}{Zilong Wang}, \bibinfo{person}{Zichao Yang}, \bibinfo{person}{Jilong Xue}, \bibinfo{person}{Lingxiao Ma}, \bibinfo{person}{Gang Cao}, {and} \bibinfo{person}{Bin Cui}.} \bibinfo{year}{2023}\natexlab{}.
\newblock \showarticletitle{Flexmoe: Scaling large-scale sparse pre-trained model training via dynamic device placement}.
\newblock \bibinfo{journal}{\emph{Proceedings of the ACM on Management of Data}} \bibinfo{volume}{1}, \bibinfo{number}{1} (\bibinfo{year}{2023}), \bibinfo{pages}{1--19}.
\newblock


\bibitem[Nie et~al\mbox{.}(2022)]%
        {Hetumoe}
\bibfield{author}{\bibinfo{person}{Xiaonan Nie}, \bibinfo{person}{Pinxue Zhao}, \bibinfo{person}{Xupeng Miao}, \bibinfo{person}{Tong Zhao}, {and} \bibinfo{person}{Bin Cui}.} \bibinfo{year}{2022}\natexlab{}.
\newblock \showarticletitle{HetuMoE: An efficient trillion-scale mixture-of-expert distributed training system}.
\newblock \bibinfo{journal}{\emph{arXiv preprint arXiv:2203.14685}} (\bibinfo{year}{2022}).
\newblock


\bibitem[Pan et~al\mbox{.}(2022)]%
        {echelonflow}
\bibfield{author}{\bibinfo{person}{Rui Pan}, \bibinfo{person}{Yiming Lei}, \bibinfo{person}{Jialong Li}, \bibinfo{person}{Zhiqiang Xie}, \bibinfo{person}{Binhang Yuan}, {and} \bibinfo{person}{Yiting Xia}.} \bibinfo{year}{2022}\natexlab{}.
\newblock \showarticletitle{Efficient flow scheduling in distributed deep learning training with echelon formation}. In \bibinfo{booktitle}{\emph{Proceedings of the 21st ACM Workshop on Hot Topics in Networks}}. \bibinfo{pages}{93--100}.
\newblock


\bibitem[Radford et~al\mbox{.}(2018)]%
        {radford2018improving}
\bibfield{author}{\bibinfo{person}{Alec Radford}, \bibinfo{person}{Karthik Narasimhan}, \bibinfo{person}{Tim Salimans}, \bibinfo{person}{Ilya Sutskever}, {et~al\mbox{.}}} \bibinfo{year}{2018}\natexlab{}.
\newblock \showarticletitle{Improving language understanding by generative pre-training}.
\newblock  (\bibinfo{year}{2018}).
\newblock


\bibitem[Rajbhandari et~al\mbox{.}(2022)]%
        {rajbhandari2022deepspeed}
\bibfield{author}{\bibinfo{person}{Samyam Rajbhandari}, \bibinfo{person}{Conglong Li}, \bibinfo{person}{Zhewei Yao}, \bibinfo{person}{Minjia Zhang}, \bibinfo{person}{Reza~Yazdani Aminabadi}, \bibinfo{person}{Ammar~Ahmad Awan}, \bibinfo{person}{Jeff Rasley}, {and} \bibinfo{person}{Yuxiong He}.} \bibinfo{year}{2022}\natexlab{}.
\newblock \showarticletitle{Deepspeed-moe: Advancing mixture-of-experts inference and training to power next-generation ai scale}. In \bibinfo{booktitle}{\emph{International conference on machine learning}}. PMLR, \bibinfo{pages}{18332--18346}.
\newblock


\bibitem[Rajbhandari et~al\mbox{.}(2020)]%
        {rajbhandari2020zero}
\bibfield{author}{\bibinfo{person}{Samyam Rajbhandari}, \bibinfo{person}{Jeff Rasley}, \bibinfo{person}{Olatunji Ruwase}, {and} \bibinfo{person}{Yuxiong He}.} \bibinfo{year}{2020}\natexlab{}.
\newblock \showarticletitle{Zero: Memory optimizations toward training trillion parameter models}. In \bibinfo{booktitle}{\emph{SC20: International Conference for High Performance Computing, Networking, Storage and Analysis}}. IEEE, \bibinfo{pages}{1--16}.
\newblock


\bibitem[Riquelme et~al\mbox{.}(2021)]%
        {riquelme2021scaling}
\bibfield{author}{\bibinfo{person}{Carlos Riquelme}, \bibinfo{person}{Joan Puigcerver}, \bibinfo{person}{Basil Mustafa}, \bibinfo{person}{Maxim Neumann}, \bibinfo{person}{Rodolphe Jenatton}, \bibinfo{person}{Andr{\'e} Susano~Pinto}, \bibinfo{person}{Daniel Keysers}, {and} \bibinfo{person}{Neil Houlsby}.} \bibinfo{year}{2021}\natexlab{}.
\newblock \showarticletitle{Scaling vision with sparse mixture of experts}.
\newblock \bibinfo{journal}{\emph{Advances in Neural Information Processing Systems}}  \bibinfo{volume}{34} (\bibinfo{year}{2021}), \bibinfo{pages}{8583--8595}.
\newblock


\bibitem[Shazeer et~al\mbox{.}(2017)]%
        {shazeer2017outrageously}
\bibfield{author}{\bibinfo{person}{Noam Shazeer}, \bibinfo{person}{Azalia Mirhoseini}, \bibinfo{person}{Krzysztof Maziarz}, \bibinfo{person}{Andy Davis}, \bibinfo{person}{Quoc Le}, \bibinfo{person}{Geoffrey Hinton}, {and} \bibinfo{person}{Jeff Dean}.} \bibinfo{year}{2017}\natexlab{}.
\newblock \showarticletitle{Outrageously large neural networks: The sparsely-gated mixture-of-experts layer}.
\newblock \bibinfo{journal}{\emph{arXiv preprint arXiv:1701.06538}} (\bibinfo{year}{2017}).
\newblock


\bibitem[Shen et~al\mbox{.}(2022)]%
        {shen2022se}
\bibfield{author}{\bibinfo{person}{Liang Shen}, \bibinfo{person}{Zhihua Wu}, \bibinfo{person}{WeiBao Gong}, \bibinfo{person}{Hongxiang Hao}, \bibinfo{person}{Yangfan Bai}, \bibinfo{person}{HuaChao Wu}, \bibinfo{person}{Xinxuan Wu}, \bibinfo{person}{Jiang Bian}, \bibinfo{person}{Haoyi Xiong}, \bibinfo{person}{Dianhai Yu}, {et~al\mbox{.}}} \bibinfo{year}{2022}\natexlab{}.
\newblock \showarticletitle{Se-moe: A scalable and efficient mixture-of-experts distributed training and inference system}.
\newblock \bibinfo{journal}{\emph{arXiv preprint arXiv:2205.10034}} (\bibinfo{year}{2022}).
\newblock


\bibitem[Shi et~al\mbox{.}(2024)]%
        {Schemoe}
\bibfield{author}{\bibinfo{person}{Shaohuai Shi}, \bibinfo{person}{Xinglin Pan}, \bibinfo{person}{Qiang Wang}, \bibinfo{person}{Chengjian Liu}, \bibinfo{person}{Xiaozhe Ren}, \bibinfo{person}{Zhongzhe Hu}, \bibinfo{person}{Yu Yang}, \bibinfo{person}{Bo Li}, {and} \bibinfo{person}{Xiaowen Chu}.} \bibinfo{year}{2024}\natexlab{}.
\newblock \showarticletitle{ScheMoE: An Extensible Mixture-of-Experts Distributed Training System with Tasks Scheduling}. In \bibinfo{booktitle}{\emph{Proceedings of the Nineteenth European Conference on Computer Systems}}. \bibinfo{pages}{236--249}.
\newblock


\bibitem[Shoeybi et~al\mbox{.}(2019)]%
        {shoeybi2019megatron}
\bibfield{author}{\bibinfo{person}{Mohammad Shoeybi}, \bibinfo{person}{Mostofa Patwary}, \bibinfo{person}{Raul Puri}, \bibinfo{person}{Patrick LeGresley}, \bibinfo{person}{Jared Casper}, {and} \bibinfo{person}{Bryan Catanzaro}.} \bibinfo{year}{2019}\natexlab{}.
\newblock \showarticletitle{Megatron-lm: Training multi-billion parameter language models using model parallelism}.
\newblock \bibinfo{journal}{\emph{arXiv preprint arXiv:1909.08053}} (\bibinfo{year}{2019}).
\newblock


\bibitem[Viswanathan et~al\mbox{.}(2020)]%
        {MLfabric}
\bibfield{author}{\bibinfo{person}{Raajay Viswanathan}, \bibinfo{person}{Arjun Balasubramanian}, {and} \bibinfo{person}{Aditya Akella}.} \bibinfo{year}{2020}\natexlab{}.
\newblock \showarticletitle{Network-accelerated distributed machine learning for multi-tenant settings}. In \bibinfo{booktitle}{\emph{Proceedings of the 11th ACM Symposium on Cloud Computing}}. \bibinfo{pages}{447--461}.
\newblock


\bibitem[Wang et~al\mbox{.}(2021)]%
        {Wavelet}
\bibfield{author}{\bibinfo{person}{Guanhua Wang}, \bibinfo{person}{Kehan Wang}, \bibinfo{person}{Kenan Jiang}, \bibinfo{person}{Xiangjun Li}, {and} \bibinfo{person}{Ion Stoica}.} \bibinfo{year}{2021}\natexlab{}.
\newblock \showarticletitle{Wavelet: Efficient DNN Training with Tick-Tock Scheduling}.
\newblock \bibinfo{journal}{\emph{Proceedings of Machine Learning and Systems}}  \bibinfo{volume}{3} (\bibinfo{year}{2021}), \bibinfo{pages}{696--710}.
\newblock


\bibitem[Wang et~al\mbox{.}(2023)]%
        {Prophet}
\bibfield{author}{\bibinfo{person}{Wei Wang}, \bibinfo{person}{Zhiquan Lai}, \bibinfo{person}{Shengwei Li}, \bibinfo{person}{Weijie Liu}, \bibinfo{person}{Keshi Ge}, \bibinfo{person}{Yujie Liu}, \bibinfo{person}{Ao Shen}, {and} \bibinfo{person}{Dongsheng Li}.} \bibinfo{year}{2023}\natexlab{}.
\newblock \showarticletitle{Prophet: Fine-grained Load Balancing for Parallel Training of Large-scale MoE Models}. In \bibinfo{booktitle}{\emph{2023 IEEE International Conference on Cluster Computing (CLUSTER)}}. IEEE, \bibinfo{pages}{82--94}.
\newblock


\bibitem[Weng et~al\mbox{.}(2022)]%
        {mlaas}
\bibfield{author}{\bibinfo{person}{Qizhen Weng}, \bibinfo{person}{Wencong Xiao}, \bibinfo{person}{Yinghao Yu}, \bibinfo{person}{Wei Wang}, \bibinfo{person}{Cheng Wang}, \bibinfo{person}{Jian He}, \bibinfo{person}{Yong Li}, \bibinfo{person}{Liping Zhang}, \bibinfo{person}{Wei Lin}, {and} \bibinfo{person}{Yu Ding}.} \bibinfo{year}{2022}\natexlab{}.
\newblock \showarticletitle{$\{$MLaaS$\}$ in the wild: Workload analysis and scheduling in $\{$Large-Scale$\}$ heterogeneous $\{$GPU$\}$ clusters}. In \bibinfo{booktitle}{\emph{19th USENIX Symposium on Networked Systems Design and Implementation (NSDI 22)}}. \bibinfo{pages}{945--960}.
\newblock


\bibitem[Wu et~al\mbox{.}(2024)]%
        {Lazarus}
\bibfield{author}{\bibinfo{person}{Yongji Wu}, \bibinfo{person}{Wenjie Qu}, \bibinfo{person}{Tianyang Tao}, \bibinfo{person}{Zhuang Wang}, \bibinfo{person}{Wei Bai}, \bibinfo{person}{Zhuohao Li}, \bibinfo{person}{Yuan Tian}, \bibinfo{person}{Jiaheng Zhang}, \bibinfo{person}{Matthew Lentz}, {and} \bibinfo{person}{Danyang Zhuo}.} \bibinfo{year}{2024}\natexlab{}.
\newblock \showarticletitle{Lazarus: Resilient and Elastic Training of Mixture-of-Experts Models with Adaptive Expert Placement}.
\newblock \bibinfo{journal}{\emph{arXiv preprint arXiv:2407.04656}} (\bibinfo{year}{2024}).
\newblock


\bibitem[Xiao et~al\mbox{.}(2018)]%
        {Gandiva}
\bibfield{author}{\bibinfo{person}{Wencong Xiao}, \bibinfo{person}{Romil Bhardwaj}, \bibinfo{person}{Ramachandran Ramjee}, \bibinfo{person}{Muthian Sivathanu}, \bibinfo{person}{Nipun Kwatra}, \bibinfo{person}{Zhenhua Han}, \bibinfo{person}{Pratyush Patel}, \bibinfo{person}{Xuan Peng}, \bibinfo{person}{Hanyu Zhao}, \bibinfo{person}{Quanlu Zhang}, {et~al\mbox{.}}} \bibinfo{year}{2018}\natexlab{}.
\newblock \showarticletitle{Gandiva: Introspective cluster scheduling for deep learning}. In \bibinfo{booktitle}{\emph{13th USENIX Symposium on Operating Systems Design and Implementation (OSDI 18)}}. \bibinfo{pages}{595--610}.
\newblock


\bibitem[Xiao et~al\mbox{.}(2020)]%
        {Antman}
\bibfield{author}{\bibinfo{person}{Wencong Xiao}, \bibinfo{person}{Shiru Ren}, \bibinfo{person}{Yong Li}, \bibinfo{person}{Yang Zhang}, \bibinfo{person}{Pengyang Hou}, \bibinfo{person}{Zhi Li}, \bibinfo{person}{Yihui Feng}, \bibinfo{person}{Wei Lin}, {and} \bibinfo{person}{Yangqing Jia}.} \bibinfo{year}{2020}\natexlab{}.
\newblock \showarticletitle{$\{$AntMan$\}$: Dynamic Scaling on $\{$GPU$\}$ Clusters for Deep Learning}. In \bibinfo{booktitle}{\emph{14th USENIX Symposium on Operating Systems Design and Implementation (OSDI 20)}}. \bibinfo{pages}{533--548}.
\newblock


\bibitem[Yi et~al\mbox{.}(2020)]%
        {heterog}
\bibfield{author}{\bibinfo{person}{Xiaodong Yi}, \bibinfo{person}{Shiwei Zhang}, \bibinfo{person}{Ziyue Luo}, \bibinfo{person}{Guoping Long}, \bibinfo{person}{Lansong Diao}, \bibinfo{person}{Chuan Wu}, \bibinfo{person}{Zhen Zheng}, \bibinfo{person}{Jun Yang}, {and} \bibinfo{person}{Wei Lin}.} \bibinfo{year}{2020}\natexlab{}.
\newblock \showarticletitle{Optimizing distributed training deployment in heterogeneous GPU clusters}. In \bibinfo{booktitle}{\emph{Proceedings of the 16th International Conference on emerging Networking EXperiments and Technologies}}. \bibinfo{pages}{93--107}.
\newblock


\bibitem[Yu et~al\mbox{.}(2024)]%
        {moesys}
\bibfield{author}{\bibinfo{person}{Dianhai Yu}, \bibinfo{person}{Liang Shen}, \bibinfo{person}{Hongxiang Hao}, \bibinfo{person}{Weibao Gong}, \bibinfo{person}{Huachao Wu}, \bibinfo{person}{Jiang Bian}, \bibinfo{person}{Lirong Dai}, {and} \bibinfo{person}{Haoyi Xiong}.} \bibinfo{year}{2024}\natexlab{}.
\newblock \showarticletitle{MoESys: A Distributed and Efficient Mixture-of-Experts Training and Inference System for Internet Services}.
\newblock \bibinfo{journal}{\emph{IEEE Transactions on Services Computing}} (\bibinfo{year}{2024}).
\newblock


\bibitem[Yu and Chowdhury(2020)]%
        {Salus}
\bibfield{author}{\bibinfo{person}{Peifeng Yu} {and} \bibinfo{person}{Mosharaf Chowdhury}.} \bibinfo{year}{2020}\natexlab{}.
\newblock \showarticletitle{Salus: Fine-Grained GPU Sharing Primitives for Deep Learning Applications}.
\newblock \bibinfo{journal}{\emph{MLSys' 20}} (\bibinfo{year}{2020}).
\newblock


\bibitem[Zhang et~al\mbox{.}(2024)]%
        {Hap}
\bibfield{author}{\bibinfo{person}{Shiwei Zhang}, \bibinfo{person}{Lansong Diao}, \bibinfo{person}{Chuan Wu}, \bibinfo{person}{Zongyan Cao}, \bibinfo{person}{Siyu Wang}, {and} \bibinfo{person}{Wei Lin}.} \bibinfo{year}{2024}\natexlab{}.
\newblock \showarticletitle{HAP: SPMD DNN Training on Heterogeneous GPU Clusters with Automated Program Synthesis}. In \bibinfo{booktitle}{\emph{Proceedings of the Nineteenth European Conference on Computer Systems}}. \bibinfo{pages}{524--541}.
\newblock


\end{thebibliography}

\clearpage
\section*{Appendix}
\appendix

\section{Proof of Theorem~\ref{thm:exclusive_homo_comm_time}}\label{appendix:theorem1}

Theorem~\ref{thm:exclusive_homo_comm_time} states that the minimum communication time with the traffic matrix $\mathbb{D}$ is given by $b_{max}$ = $max(\sum_{j=1}^{n}d_{ij},\; \sum_{i=1}^{n}d_{ij}) / B$. Here, $d_{ij}$ is the element located at row $i$ and column $j$ in $\mathbb{D}$, and $B$ denotes the bandwidth for each homogeneous GPU.

For simplicity, we set bandwidth $B$ to 1. Our approach unfolds in three key steps. Initially, we illustrate the conversion of the traffic matrix $\mathbb{D}$ into $\mathbb{D'}$ by incorporating matrix $\mathbb{X}$. Subsequently, we prove that the minimum communication time for $\mathbb{D'}$ is $b_{max}$. Finally, we prove the existence of a non-negative $\mathbb{X}$.

\vspace{2mm}
\para{\textbf{1. Convert $\mathbb{D}$ to $\mathbb{D'}$ by adding non-negative $\mathbb{X}$}}

\begin{table*}[htbp]
\centering
\begin{minipage}{\textwidth}
\begin{equation}\label{eqn:add_x}
\begin{bNiceMatrix}[last-row, last-col]
d_{11} & d_{12} & ... & d_{1n}  & b_{1}\\
d_{21} & d_{22} & ... & d_{2n}  & b_{2}\\
... & ... & ... & ...           & \vdots\\
d_{n1} & d_{n2} & ... & d_{nn}  & b_{n}\\
b_{n+1} & b_{n+2} & ... & b_{2n}  \\
\end{bNiceMatrix}  + \;
\begin{bNiceMatrix}[last-row, last-col]
x_{11} & x_{12} & ... & x_{1n}  & \Delta b_{1}\\
x_{21} & x_{22} & ... & x_{2n}  & \Delta b_{2}\\
... & ... & ... & ...           & \vdots\\
x_{n1} & x_{n2} & ... & x_{nn}  & \Delta b_{n}\\
\Delta b_{n+1} & \Delta b_{n+2} & ... & \Delta b_{2n}  \\
\end{bNiceMatrix}  = 
\begin{bNiceMatrix}[last-row, last-col]
d'_{11} & d'_{12} & ... & d'_{1n}  & b_{max}\\
d'_{21} & d'_{22} & ... & d'_{2n}  & b_{max}\\
... & ... & ... & ...           & \vdots\\
d'_{n1} & d'_{n2} & ... & d'_{nn}  & b_{max}\\
b_{max} & b_{max} & ... & b_{max}  \\
\end{bNiceMatrix}
\end{equation}
\end{minipage}
\end{table*}

Eqn.~\ref{eqn:add_x} illustrates the relationship $\mathbb{D}$ + $\mathbb{X}$ = $\mathbb{D'}$, where $x_{ij}$ and $d'_{ij}$ denote elements located at row $i$ and column $j$ in $\mathbb{X}$ and $\mathbb{D'}$, respectively. Values external to the matrices, such as $b_{1}$, $\Delta b_{1}$, and $b_{max}$, represent the sum of their corresponding columns or rows.  For the traffic matrix \(\mathbb{D'}\), the conditions are met such that the sum of each row $\sum_{j=1}^{n}d'_{ij} = b_{max}$ and the sum of each column $\sum_{i=1}^{n}d'_{ij} = b_{max}$. That is to say, each GPU is sending and receiving precisely $b_{max}$ traffic after adding artificial traffic matrix $\mathbb{X}$. As shown in Fig.~\ref{fig:theorem_1}(a), each GPU sends and receives $b_{max}$ traffic in total. The label attached to each traffic entry indicates the target GPU to which the traffic is directed.

\vspace{2mm}
\para{\textbf{2. Prove the minimum communication time for $\mathbb{D'}$ is $b_{max}$}}

\begin{wrapfigure}{r}{0.5\textwidth}
    \centering
    \includegraphics[width=1\linewidth]{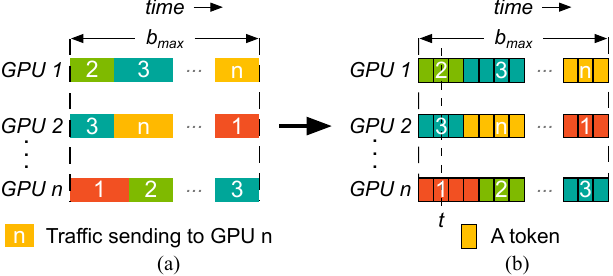}%
    \caption{(a) Each GPU sends/receives $b_{max}$ traffic in total. (b) Each GPU receives only one token at a time.}\label{fig:theorem_1}
    \vspace{-0.05in}
\end{wrapfigure}

Now, our attention shifts to determining the minimum communication time for  $\mathbb{D'}$. To establish that the minimum communication time for $\mathbb{D'}$ is indeed $b_{max}$, it is imperative to demonstrate that each GPU is capable of transmitting and receiving traffic without any interruptions, until all the traffic is completely finished. Any interruption would necessarily result in a communication time exceeding $b_{max}$, given that each GPU is expected to both send and receive a total of $b_{max}$ traffic.

As depicted in Fig.~\ref{fig:theorem_1}(b), one \textit{time slot} is required to transmit a token at full bandwidth. At any given time slot $t$, each GPU is configured to transmit just one token at its full bandwidth. As a result, each GPU can only receive one token at the same time. This is attributed to the parity in bandwidth between the sending and receiving sides, where the receiving side cannot simultaneously accommodate two tokens sent by two GPUs at full bandwidth.

We then proceed to establish that each GPU can transmit and receive tokens without any disruptions until all tokens are completely delivered. At time slot $t$, we can identify the presence of $n$ tokens, one originating from each GPU, with each destined for a distinct GPU among the $n$ GPUs. These $n$ tokens are systematically labeled from 1 to $n$. To verify this, we employ a proof by contradiction. We assume the hypothetical scenario where no token is directed to GPU $i$ (i.e., a token labeled with $i$) at time slot $t$. This assumption leads to the conclusion that the receiving traffic of GPU $i$ cannot reach the stipulated value of $b_{max}$. This, however, contradicts the requirement that each GPU must receive $b_{max}$ traffic under the traffic matrix $\mathbb{D'}$. Given the presence of $n$ tokens heading to $n$ distinct GPUs, these tokens can be transmitted without any contention during time slot $t$. This process can be iterated until all tokens are successfully transmitted. In other words, under the traffic matrix $\mathbb{D'}$, all GPUs participate in a seamless exchange of traffic without any interruptions. The minimum communication time for $\mathbb{D'}$ is firmly established as $b_{max}$.

\begin{equation}\label{eqn:equation_x}
\begin{aligned}
\left\{\begin{array}{cll}
     x_{11} + x_{12} + ... + x_{1n}  &= b_{max}-b_{1} &= \Delta b_{1}\\
     x_{21} + x_{22} + ... + x_{2n}  &= b_{max}-b_{2} &= \Delta b_{2}\\
     \vdots & \vdots & \vdots\\
     x_{n1} + x_{n2} + ... + x_{nn}  &= b_{max}-b_{n} &= \Delta b_{n}\\
     x_{11} + x_{21} + ... + x_{n1}  &= b_{max}-b_{n+1} &= \Delta b_{n+1}\\
     \vdots & \vdots & \vdots\\
     x_{1n} + x_{2n} + ... + x_{nn}  &= b_{max}-b_{2n} &= \Delta b_{2n}\\
\end{array}\right.
\end{aligned}
\end{equation}

\begin{equation}\label{eqn:equation_x_matrix}
\begin{bNiceMatrix}[first-row, first-col]
& 1 & 2 & \cdots & n & n+1 & \cdots & n^2 &\\
1 & 1 & 1 & \cdots & 1 &0 &\cdots &0\\
2 & 0 & 0 & \cdots & 0 &1 &\cdots &0\\
\vdots & \cdots & \cdots & \cdots &\cdots &\cdots &\cdots &\cdots\\
n & 0 & 0 & \cdots & 0 &0 &\cdots &1\\
 & 1 & 0 & \cdots & 0 &1 &\cdots &0\\
\vdots & \cdots & \cdots & \cdots &\cdots &\cdots &\cdots& \cdots\\
2n & 0 & 0 & \cdots & 1 &0 &\cdots&1\\
\end{bNiceMatrix}  
\begin{bNiceMatrix}
x_{11}\\
x_{12}\\
x_{13}\\
\vdots \\
x_{1n}\\
x_{21}\\
\vdots \\
x_{nn}\\
\end{bNiceMatrix}  = 
\begin{bNiceMatrix}
\Delta b_{1}\\
\Delta b_{2}\\
\vdots \\
\Delta b_{n}\\
\vdots \\
\Delta b_{2n}\\
\end{bNiceMatrix}
\end{equation}

\vspace{2mm}
\para{\textbf{3. Prove the existence of non-negative $\mathbb{X}$}}

In this step, we need to prove the presence of a non-negative matrix $\mathbb{X}$. The existence of a non-negative $\mathbb{X}$ is a critical factor in ensuring that the minimum communication time of $\mathbb{D}$ does not exceed that of $\mathbb{D'}$.

Eqn.~\ref{eqn:equation_x} presents the equations that elements of $\mathbb{X}$ should satisfy. We transform the $n \times n$ matrix $\mathbb{X}$ to an $n^{2} \times 1$ vector $\mathbf{x}$. Then we express these equations in matrix format as $\mathbb{A} \mathbf{x} = \Delta \mathbf{b}$, which is shown in Eqn.~\ref{eqn:equation_x_matrix}. The size of $\mathbb{A}$ is $2n \times n^2$. Notably, it is apparent that for every $\Delta b_{i} \in $ $\Delta \mathbf{b}$, we have $\Delta b_{i} \geq 0$ since $b_{max} \geq b_{i}$. Next, we use Farkas' Lemma~\cite{farkas} to prove the existence of a non-negative solution $\mathbf{x}$.

\vspace{2mm}
\textit{Farkas' Lemma~\cite{farkas}: Let $\mathbb{A} \in \mathbb{R}^{m\times n}$ and $\mathbf{b} \in \mathbb{R}^{m}$. Then exactly one of the following two assertions is true:\\
\indent \indent 1. There exists an $\mathbf{x} \in \mathbb{R}^{n}$ such that $\mathbb{A} \mathbf{x} =\mathbf {b}$ and $\mathbf {x} \geq 0$.\\
\indent \indent 2. There exists a $\mathbf{y} \in \mathbb{R}^{m}$ such that $\mathbb{A}^{\mathsf{T}}\mathbf{y} \geq 0$ and $\mathbf{b} ^{\mathsf{T}}\mathbf{y} <0$.\\}
\vspace{0mm}

Here, the notation $\mathbf {x} \geq 0$ means that all components of the vector $\mathbf{x}$ are non-negative.

\begin{equation}\label{eqn:equation_y}
\begin{aligned}
\left\{\begin{array}{cll}
     y_{1} + y_{n+1}  \geq 0 \\
     y_{1} + y_{n+2}  \geq 0 \\
       \vdots \\
     y_{1} + y_{2n}  \geq 0 \\
\end{array}\right. 
\left\{\begin{array}{cll}
     y_{2} + y_{n+1}  \geq 0 \\
     y_{2} + y_{n+2}  \geq 0 \\
      \vdots \\
     y_{2} + y_{2n}  \geq 0 \\
\end{array}\right.
 \vdots 
\left\{\begin{array}{cll}
     y_{n} + y_{n+1}  \geq 0 \\
     y_{n} + y_{n+2}  \geq 0 \\
       \vdots \\
     y_{n} + y_{2n}  \geq 0 \\
\end{array}\right.
\end{aligned}
\end{equation}

\begin{equation}\label{eqn:y_min}
\begin{aligned}
\left\{\begin{array}{ccc}
     y_{n+1} &\geq -y_{1}, -y_{2},\cdots, -y_{n}  \\
     y_{n+2} &\geq -y_{1}, -y_{2},\cdots, -y_{n}  \\
     \vdots & \vdots \\
     y_{2n} &\geq -y_{1}, -y_{2},\cdots, -y_{n}  \\
\end{array}\right.  \Rightarrow 
\left\{\begin{array}{ccc}
     y_{n+1} &\geq -y_{min}  \\
     y_{n+2} &\geq -y_{min} \\
     \vdots & \vdots \\
     y_{2n} &\geq -y_{min}  \\
\end{array}\right.
\end{aligned}
\end{equation}

Assertion 1 aligns precisely with our objective. To affirm Assertion 1, we need to disprove Assertion 2. This can be achieved through a proof by contradiction. Assume Assertion 2 is true: there exists a $\mathbf{y}$ with size $2n \times 1$ such that $\mathbb{A}^{\mathsf{T}}\mathbf{y} \geq 0$ and $\mathbf{b}^{\mathsf{T}}\mathbf{y} <0$. By applying $\mathbb{A}^{\mathsf{T}}\mathbf{y} \geq 0$, we derive the inequalities as shown in Eqn.~\ref{eqn:equation_y}.

Assume $y_{min} = min(y_{1}, y_{2}, ... , y_{n})$, we have Eqn.~\ref{eqn:y_min}. Then we calculate the value of $\Delta \mathbf{b}^{\mathsf{T}}\mathbf{y}$.

\begin{equation}\label{eqn:contradiction}
\begin{aligned}
\Delta \mathbf{b}^{\mathsf{T}}\mathbf{y} &= \Delta b_{1}y_{1} + \cdots + \Delta b_{n}y_{n}+ \Delta b_{n+1}y_{n+1} \\
& \quad + \cdots+ \Delta b_{2n}y_{2n} \\
&\geq \Delta b_{1}y_{min} + \cdots+ \Delta b_{n}y_{min} + \Delta b_{n+1}(-y_{min}) \\
& \quad + \cdots + \Delta b_{2n}(-y_{min}) \\
&= y_{min}((\Delta b_{1} + \cdots \Delta b_{n}) - (\Delta b_{n+1} + \cdots + \Delta b_{2n})) \\
&=0
\end{aligned}
\end{equation}

From Eqn.~\ref{eqn:contradiction}, we know $\Delta \mathbf{b}^{\mathsf{T}}\mathbf{y} \geq 0$. This contradicts $\Delta \mathbf{b}^{\mathsf{T}}\mathbf{y} < 0$ in Assertion 2, proving that Assertion 2 is incorrect. As a result, we can establish the existence of a non-negative solution $\mathbf{x}$, and this, in turn, confirms the presence of a non-negative matrix~$\mathbb{X}$.

\section{Proof of Theorem~\ref{thm:exclusive_hetero_comm_time}}\label{appendix:theorem2}

Theorem~\ref{thm:exclusive_hetero_comm_time} demonstrates that the minimum communication time with traffic matrix $\mathbb{D}$ is $b_{max}$ = $max(\sum_{j=1}^{n}d_{ij}/ B_i,\; \sum_{i=1}^{n}d_{ij}/ B_i)$, where $d_{ij}$ is the element located at row $i$ and column $j$ in $\mathbb{D}$, $B_i$ is the bandwidth of GPU $i$.

Similar to the proof of Theorem~\ref{thm:exclusive_homo_comm_time}, the approach unfolds in three steps. Initially, we convert the traffic matrix $\mathbb{D}$ into $\mathbb{D'}$ using matrix $\mathbb{X}$. Next, we demonstrate that the minimum communication time for $\mathbb{D'}$ is $b_{max}$. Finally, we prove the existence of a non-negative $\mathbb{X}$. The key difference is that the network bandwidth $B$ cannot be simplified to 1 in a heterogeneous environment. This distinction must be incorporated into the proof.

\vspace{2mm}
\para{\textbf{1. Convert $\mathbb{D}$ to $\mathbb{D'}$ by adding non-negative $\mathbb{X}$}}

\begin{table*}[htbp]
\centering
\begin{minipage}{\textwidth}
\begin{equation}\label{eqn:add_x_2}
\begin{aligned}
\begin{bNiceMatrix}[last-row, last-col]
\frac{d_{11}}{B_{1}} & \frac{d_{12}}{B_{1}} & \cdots & \frac{d_{1n}}{B_{1}} & b_{1}\\
\frac{d_{21}}{B_{2}} & \frac{d_{22}}{B_{2}} & \cdots & \frac{d_{2n}}{B_{2}} & b_{2}\\
\vdots & \vdots & \ddots & \vdots & \vdots\\
\frac{d_{n1}}{B_{n}} & \frac{d_{n2}}{B_{n}} & \cdots & \frac{d_{nn}}{B_{n}} & b_{n}\\
b_{n+1} & b_{n+2} & \cdots & b_{2n} \\
\end{bNiceMatrix}
+ 
\begin{bNiceMatrix}[last-row, last-col]
\frac{x_{11}}{B_{1}} & \frac{x_{12}}{B_{1}} & \cdots & \frac{x_{1n}}{B_{1}} & \Delta b_{1}\\
\frac{x_{21}}{B_{2}} & \frac{x_{22}}{B_{2}} & \cdots & \frac{x_{2n}}{B_{2}} & \Delta b_{2}\\
\vdots & \vdots & \ddots & \vdots & \vdots\\
\frac{x_{n1}}{B_{n}} & \frac{x_{n2}}{B_{n}} & \cdots & \frac{x_{nn}}{B_{n}} & \Delta b_{n}\\
\Delta b_{n+1} & \Delta b_{n+2} & \cdots & \Delta b_{2n} \\
\end{bNiceMatrix}
= 
\begin{bNiceMatrix}[last-row, last-col]
\frac{d'_{11}}{B_{1}} & \frac{d'_{12}}{B_{1}} & \cdots & \frac{d'_{1n}}{B_{1}} & b_{max}\\
\frac{d'_{21}}{B_{2}} & \frac{d'_{22}}{B_{2}} & \cdots & \frac{d'_{2n}}{B_{2}} & b_{max}\\
\vdots & \vdots & \ddots & \vdots & \vdots\\
\frac{d'_{n1}}{B_{n}} & \frac{d'_{n2}}{B_{n}} & \cdots & \frac{d'_{nn}}{B_{n}} & b_{max}\\
b_{max} & b_{max} & \cdots & b_{max} \\
\end{bNiceMatrix}
\end{aligned}
\end{equation}
\end{minipage}
\end{table*}

To address the difference, we modify the element $d_{ij}$ in $\mathbb{D}$ to $d_{ij} /\min(B_i, B_j) $, as indicated in Eqn.~\ref{eqn:add_x_2}. We apply the same adjustment to the elements in $\mathbb{X}$ and $\mathbb{D'}$. For clarity and consistency, we continue to refer to these updated matrices as $\mathbb{D}$, $\mathbb{X}$, and $\mathbb{D'}$.

Eqn.~\ref{eqn:add_x_2} illustrates the relationship $\mathbb{D}$ + $\mathbb{X}$ = $\mathbb{D'}$, where $x_{ij}/B_i$ and $d'_{ij}/B_i$ denote elements located at row $i$ and column $j$ in $\mathbb{X}$ and $\mathbb{D'}$, respectively. Values external to the matrices, such as $b_{1}$, $\Delta b_{1}$, and $b_{max}$, represent the sum of their corresponding columns or rows. For traffic matrix $\mathbb{D'}$, it satisfies the conditions that for each row $\sum_{j=1}^{n}d'_{ij}/B_i = b_{max}$, and for each column $ \sum_{i=1}^{n}d'_{ij}/B_i = b_{max}$. That is to say, the time each GPU uses for sending and receiving is precisely $b_{max}$ after adding artificial matrix $\mathbb{X}$. 

\vspace{2mm}
\para{\textbf{2. Prove the minimum communication time for $\mathbb{D'}$ is $b_{max}$}}

To establish that the minimum communication time for $\mathbb{D'}$ is $b_{max}$, we must demonstrate that each GPU can continuously send and receive traffic until all traffic is completed. Similar to Theorem~\ref{thm:exclusive_homo_comm_time}, each token requires a time slot to transmit at full bandwidth. Due to bandwidth constraints, each GPU can only send and receive one token per slot. To prove this, assume by contradiction that any GPU does not receive a token in a given slot. This failure means it won't meet the $b_{max}$ requirement, contradicting the need for each GPU to receive $b_{max}$ traffic. Using the same method as in Theorem~\ref{thm:exclusive_homo_comm_time}, we confirm uninterrupted traffic exchange, thereby proving that $b_{max}$ is indeed the minimum communication time for $\mathbb{D'}$.

\vspace{2mm}
\para{\textbf{3. Prove the existence of non-negative $\mathbb{X}$}}

In this step, we need to prove the presence of a non-negative matrix $\mathbb{X}$. This step is exactly the same as the one in Appx.~\ref{appendix:theorem1}. We can still use Farkas' Lemma~\cite{farkas} to prove the existence of a non-negative solution $\mathbf{x}$.

These three steps validate the correctness of Theorem~\ref{thm:exclusive_hetero_comm_time}.

\end{document}